\documentclass[nohyperref]{article}

\usepackage{amsmath,amsfonts,amsthm,bm}

\def\eqref#1{equation~\ref{#1}}

\def\1{\bm{1}}

\def\va{{\bm{a}}}
\def\vb{{\bm{b}}}

\def\ve{{\bm{e}}}
\def\vf{{\bm{f}}}
\def\vg{{\bm{g}}}
\def\vh{{\bm{h}}}

\def\vs{{\bm{s}}}

\def\vu{{\bm{u}}}
\def\vv{{\bm{v}}}
\def\vw{{\bm{w}}}
\def\vx{{\bm{x}}}

\def\vz{{\bm{z}}}

\def\mA{{\bm{A}}}
\def\mB{{\bm{B}}}

\def\mD{{\bm{D}}}

\def\mG{{\bm{G}}}
\def\mH{{\bm{H}}}

\def\mM{{\bm{M}}}

\def\mS{{\bm{S}}}

\def\mU{{\bm{U}}}
\def\mV{{\bm{V}}}
\def\mW{{\bm{W}}}
\def\mX{{\bm{X}}}

\DeclareMathAlphabet{\mathsfit}{\encodingdefault}{\sfdefault}{m}{sl}
\SetMathAlphabet{\mathsfit}{bold}{\encodingdefault}{\sfdefault}{bx}{n}
\newcommand{\tens}[1]{\bm{\mathsfit{#1}}}
\def\tA{{\tens{A}}}

\def\tM{{\tens{M}}}

\newcommand{\R}{\mathbb{R}}

\usepackage[table]{xcolor}

\usepackage{microtype}
\usepackage{graphicx}
\usepackage{booktabs} 

\usepackage{hyperref}

\usepackage[accepted]{icml2022}

\usepackage{amsmath}
\usepackage{amssymb}
\usepackage{mathtools}
\usepackage{amsthm}

\usepackage{blkarray}
\usepackage{nicematrix}
\usepackage{xurl}
\usepackage{subcaption}
\usepackage{bbm}
\usepackage{lipsum}
\usepackage{stmaryrd}
\usepackage{thmtools}

\usepackage[capitalize,noabbrev]{cleveref}

\theoremstyle{plain}
\newtheorem{theorem}{Theorem}[section]
\newtheorem{proposition}[theorem]{Proposition}
\newtheorem{lemma}[theorem]{Lemma}
\newtheorem{corollary}[theorem]{Corollary}
\newtheorem{informal}{Theorem}
\theoremstyle{definition}
\newtheorem{definition}[theorem]{Definition}
\newtheorem{assumption}[theorem]{Assumption}
\theoremstyle{remark}
\newtheorem{remark}[theorem]{Remark}

\newcommand{\C}{\mathbb{C}}
\newcommand{\mbeta}{{{}\widehat{\boldsymbol{\beta}}^\infty}}
\newcommand{\vbeta}{{\boldsymbol{\beta}}}
\newcommand{\mbetdag}{{{}\widehat{\vbeta}^{\infty}}^\dagger}
\newcommand{\vrho}{{\boldsymbol{\rho}}}
\newcommand{\mbet}{\boldsymbol{{\beta}}}
\newcommand{\mwl}[1]{{{}\boldsymbol{\widehat{w}}_{#1}^\infty}}
\newcommand{\mwld}[1]{{{{}\boldsymbol{\widehat{w}}_{#1}^\infty}^\dagger}}
\newcommand{\myz}{{{}\widehat{\boldsymbol{z}}}}
\newcommand{\myW}{\boldsymbol{W}}
\newcommand{\myx}{{{}\widehat{\boldsymbol{x}}}}

\newcommand{\cF}{\mathcal{F}}

\newcommand{\cP}{\mathcal{P}}

\newcommand{\on}{\operatorname}
\DeclareMathOperator{\tr}{tr}

\newcommand{\mathify}[1]{\ifmmode{#1}\else\mbox{$#1$}\fi}
\newcommand{\abs}[1]{\mathify{\left| #1 \right|}}

\newcommand\aaa{\cellcolor{cyan!90}}
\newcommand\bbb{\cellcolor{orange!90}}
\newcommand\ccc{\cellcolor{green!90}}

\usepackage[textsize=tiny]{todonotes}

\icmltitlerunning{Implicit Bias of Linear Equivariant Networks}

\begin{document}

\twocolumn[
\icmltitle{Implicit Bias of Linear Equivariant Networks}




\begin{icmlauthorlist}
\icmlauthor{Hannah Lawrence}{MITEECS}
\icmlauthor{Kristian Georgiev}{MITEECS}
\icmlauthor{Andrew Dienes}{MITEECS}
\icmlauthor{Bobak T. Kiani}{MITEECS}
\end{icmlauthorlist}

\icmlaffiliation{MITEECS}{Department of Electrical Engineering and Computer Science, Massachusetts Institute of Technology, Cambridge, MA 02139, USA}

\icmlcorrespondingauthor{Bobak T. Kiani}{bkiani@mit.edu}

\icmlkeywords{Machine Learning, ICML}

\vskip 0.3in
]



\printAffiliationsAndNotice{}  

\begin{abstract}
Group equivariant convolutional neural networks (G-CNNs) are generalizations of convolutional neural networks (CNNs) which excel in a wide range of technical applications by explicitly encoding symmetries, such as rotations and permutations, in their architectures. Although the success of G-CNNs is driven by their \emph{explicit} symmetry bias, a recent line of work has proposed that the \emph{implicit} bias of training algorithms on particular architectures is key to understanding generalization for overparameterized neural nets. In this context, we show that $L$-layer full-width linear G-CNNs trained via gradient descent for binary classification converge to solutions with low-rank Fourier matrix coefficients, regularized by the $2/L$-Schatten matrix norm. Our work strictly generalizes previous analysis on the implicit bias of linear CNNs to linear G-CNNs over all finite groups, including the challenging setting of non-commutative groups (such as permutations), as well as band-limited G-CNNs over infinite groups. We validate our theorems via experiments on a variety of groups, and empirically explore more realistic nonlinear networks, which locally capture similar regularization patterns. Finally, we provide intuitive interpretations of our Fourier space implicit regularization results in real space via uncertainty principles. 
\end{abstract}


\section{Introduction}\label{sec:introduction}

Modern deep learning algorithms typically have many more parameters than data points, and their ability to achieve good generalization in this overparameterized setting is largely unexplained by current theory. Classic generalization bounds, which bound the generalization error when models are not overly ``complex," are vacuous for neural networks that can perfectly fit random training labels \citep{zhang2017generalization}. More recent work analyzes the complexity of deep learning algorithms by instead characterizing the properties of the functions 
they output. Notably, prior work has shown that training via gradient descent implicitly regularizes towards certain hypothesis classes with low complexity, which may generalize better as a result. For example, in underdetermined least squares regression, gradient descent converges to the $\ell_2$-norm minimizer, while a pointwise-square reparametrization converges to the $\ell_1$-norm minimizer \citep{gunasekar2018geometry}. For separable linear regression, \citet{soudry2018separable} proved that the learned predictor under gradient descent converges in direction to the max-margin solution. Such phenomena are consistent with certain linear neural networks, \textit{e.g.,} \citet{lyu2020gradient} extended this max-margin result to gradient descent on any homogeneous neural network, and \citet{gunasekar2018linearconv} showed that learned linear fully-connected and convolutional networks implicitly regularize the $\ell_2$ norm and a depth-dependent norm in Fourier space, respectively.


From a more applied perspective, a large body of work imposes structured inductive biases on deep learning algorithms to exploit symmetry patterns~\citep{kondor2007novel, reisert2008group, group-equivariant-taco}.
One prominent method parameterizes models over functions that are \textit{equivariant} with respect to a symmetry group (\textit{i.e.,} outputs transform predictably in response to input transformation). In fact, \citet{kondor2018generalization} showed that any group equivariant network can be expressed as a series of group convolutional layers interwoven with pointwise nonlinearities, demonstrating that group convolutional neural networks (G-CNNs) are the most general family of equivariant networks.

\begin{figure*}[t!]
    \captionsetup[subfigure]{aboveskip=-1pt,belowskip=-3pt}  
  \begin{subfigure}{0.49 \textwidth}
    \includegraphics[width=\linewidth]{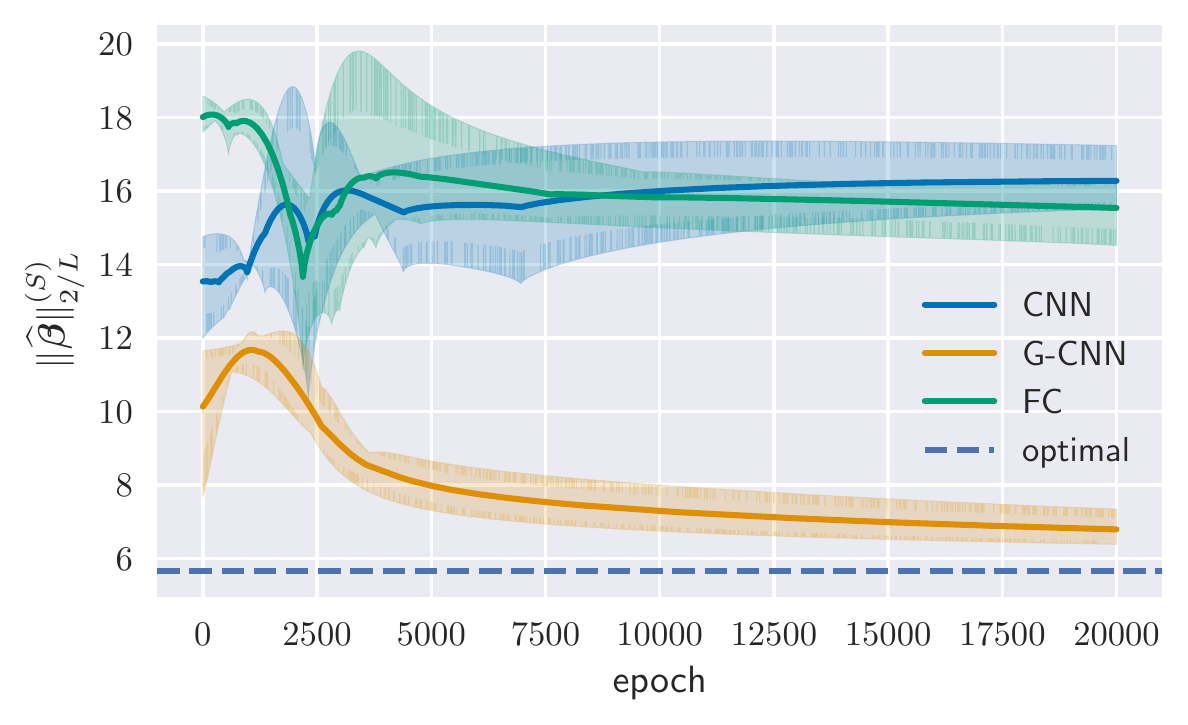}
    \caption{Fourier space norm of network linearization $\vbeta$} \label{fig:d8_gaussian_2_a}
  \end{subfigure}%
  \hspace*{\fill}   
  \begin{subfigure}{0.49\textwidth}
    \includegraphics[width=\linewidth]{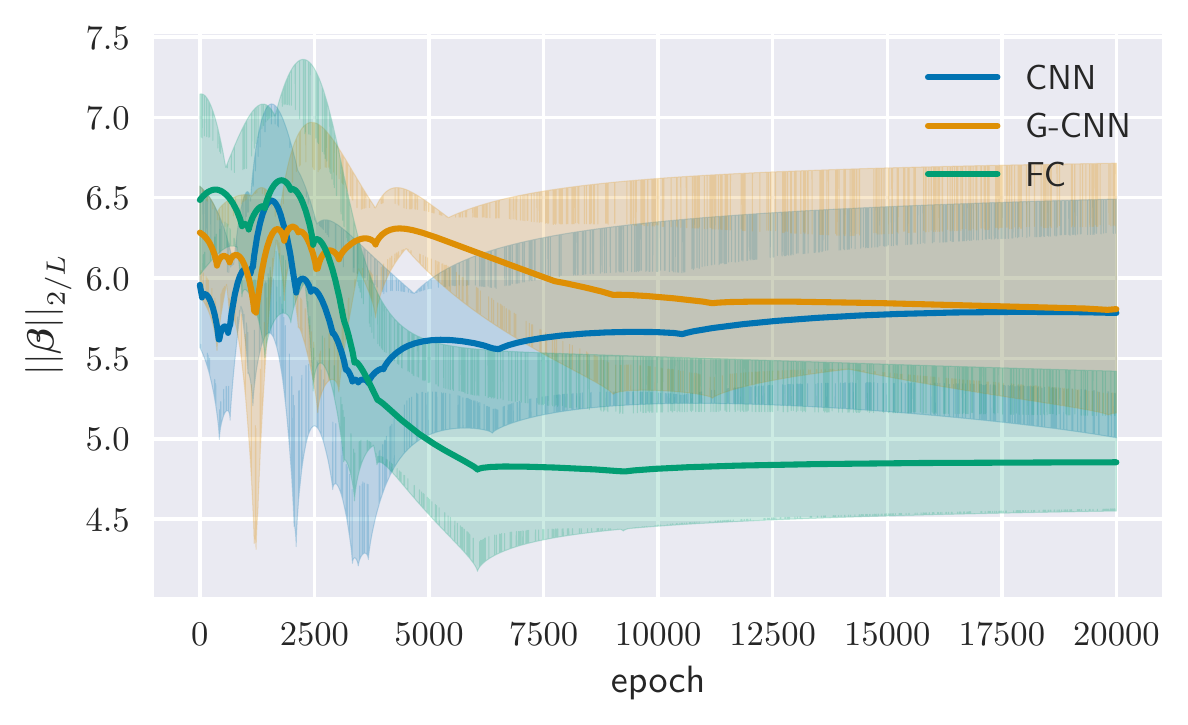}
    \caption{Real space norm of network linearization $\vbeta$} \label{fig:d8_gaussian_2_b}
  \end{subfigure}%
  \hspace*{\fill}
\caption{Training via gradient descent of linear G-CNNs (with linearization $\vbeta$) implicitly biases towards sparse singular values for Fourier matrix coefficients of $\boldsymbol{\beta}$ (see \autoref{thm:main_informal}). \autoref{fig:d8_gaussian_2_a} traces the $D_8$ Fourier space sparsity over the course of training three architectures in an overparameterized classification task, beginning from epoch 1. The G-CNN converges to a Fourier-sparse linearization, in contrast with fully-connected and convolutional networks. The horizontal line shows the minimum Fourier Schatten norm among all interpolating predictors, obtained as the solution of a convex relaxation; since the G-CNN norm converges toward this value, this verifies our main statement of convergence to a local minimum. \autoref{fig:d8_gaussian_2_b} illustrates the uncertainty principles, discussed in \autoref{sec:experiments} and illustrated more explicitly in \autoref{fig:d60_viz_linearization} that sparseness in (group) Fourier space necessitates being ``dense" in real space, and vice versa.} 

\label{fig:d8_gaussian_2}
\end{figure*}

\textbf{Our Contributions:} The explicit inductive bias of G-CNNs is the main reason for their usage. Yet, to the best of our knowledge, the implicit bias imposed by equivariant architectures has not been explored. Here, we greatly generalize the results of ~\citet{yun2020unifying} and ~\citet{gunasekar2018linearconv} to linear G-CNNs whose hidden layers perform group equivariant transformations. We show, surprisingly, that $L$-layer G-CNNs are implicitly regularized by the $2/L$-Schatten norm, which is the $2/L$ norm of a matrix's singular values, over the irreducible representations in the Fourier basis. As a result, convergence is biased towards sparse solutions in the Fourier regime, as summarized in \autoref{thm:main_informal} and illustrated in \autoref{fig:d8_gaussian_2_a}, as well as in further experiments in \autoref{sec:experiments}.

\textbf{New Technical Ingredients: } Our primary technical contribution is to generalize the proof technique of \citet{gunasekar2018linearconv} by realizing it in a more general setting, translating it to non-abelian groups using the language of representation theory. Since convolutions in real space correspond to (non-commutative) matrix multiplications, rather than scalar multiplications, in the appropriate group Fourier space, this substantially complicates analysis of KKT conditions. Nonetheless, we use the framework of non-commutative Fourier and convex analysis to prove that a particular function on the singular values of the Fourier-space linearization is being implicitly optimized. 



\begin{informal}[main result; informal]\label{thm:main_informal}
Let $\on{NN}_{L}(\cdot)$ denote an $L$ layer linear group convolutional neural network, in which each hidden layer performs group cross-correlation over the group $G$ with a full-width kernel, and the final layer is a fully connected layer. When learning linearly separable data $\{\vx_i, y_i\}_{i=1}^n$ using the exponential loss function, the network converges in direction to a linear function $\on{NN}_L(\vx) \propto \boldsymbol{\beta}^T \vx $, where $\boldsymbol{\beta}$ is given by a stationary point of the following minimization problem:
\begin{equation}
    \min_{\boldsymbol{\beta}} \left\| \widehat{\boldsymbol{\beta}} \right\|_{2/L}^{(S)} \; \; \; \text{s.t.} \; \; \; y_i \vx_i^T \boldsymbol{\beta} \geq 1 \; \; \forall i \in [n]
\end{equation}
Here, $\left\| \widehat{\boldsymbol{\beta}} \right\|_{2/L}^{(S)}$ denotes the $2/L$-Schatten norm of the matrix Fourier transform of $\boldsymbol{\beta}$ (see \autoref{def:group_fourier_transform}) equivalent to 
\begin{equation}
    \left\| \widehat{\boldsymbol{\beta}} \right\|_{2/L}^{(S)} = \left[ \sum_{\rho \in {}\widehat{G}} d_\rho \left( \left\| {}\widehat{\boldsymbol{\beta}}(\rho)  \right\|^{(S)}_{2/L} \right)^{2/L} \right]^{L/2},
\end{equation}
where ${}\widehat{G}$ is a complete set of unitary irreducible representations of $G$ and $d_\rho$ is the dimension of irreducible representation $\rho$.
\end{informal}

We note that both sparsity (for vectors) and low-rankness (for matrices) are desirable properties for many applications, not least because such predictors are efficient to store and manipulate, thus potentially expanding the scope of application of G-CNNs to areas where sparsity and low-rankness are explicit desiderata. 
Connecting our findings to research on uncertainty theorems \citep{wigderson2021uncertainty}, we also show that the implicit regularization towards sparseness (or low rank irreducible representations) in the Fourier regime necessarily implies that solutions in the real regime are ``dense," as illustrated in \autoref{fig:d8_gaussian_2_b}. These results provide a more intuitive and practical perspective into the inductive bias of G-CNNs and the types of functions that they learn.

We proceed as follows. \Autoref{sec:related_work} discusses related works and their relation to our contributions. In \Autoref{sec:notation}, we define notation.  \autoref{sec:grouptheorybackground} provides a basic background in the group theory and Fourier analysis necessary to understand our results, with a more complete exposition in \autoref{app:gfourier}. Our main results are stated in \autoref{sec:main_results}, with the main proof ideas for the abelian (or commutative) and non-abelian (or non-commutative) cases given in \autoref{subsec:abelian} and \autoref{subsec:nonabelian}, respectively (complete proofs can be found in \autoref{sec:proofs}). \autoref{sec:experiments} validates our theoretical results with synthetic experiments on a variety of groups, and exploratory experiments validating our theory on non-linear networks. Finally, we discuss these results and future questions in \autoref{sec:discussion}.

\section{Related Work}\label{sec:related_work}
Enforcing equivariance and symmetries via parameter sharing schemes was introduced in the group theoretic setting in~\citet{cohen2016gcnn} and~\citet{gens2014deep}. Despite considerable interest in equivariant learning, no works to our knowledge have explored the implicit regularization of gradient descent on equivariant convolutional neural networks. 
We show that the tensor formulation of neural networks in~\citet{yun2020unifying} and the proofs in~\citet{gunasekar2018linearconv} encompass G-CNNs for which the underlying group is cyclic, and we naturally extend their results to G-CNNs over any commutative group (see \autoref{subsec:abelian}).
However, these works do not cover the case of convolutions with respect to non-commutative groups, such as three-dimensional rotations and permutations, which incidentally include some of the most compelling applications of group equivariance in practice~\citep{zaheer2017deep, anderson2019cormorant, Esteves_2018_ECCV}. 
As such, articulating the implicit bias in the more general non-abelian case is important for understanding many of the current group equivariant architectures~\citep{zaheer2017deep, kondor2018clebsch, Esteves_2018_ECCV, cesa2019e2cnn}. 
Non-abelian convolutions require more structure to theoretically analyze compared to abelian convolutions: the former are merely pointwise multiplications in Fourier space, whereas the latter are \emph{matrix} multiplications between irreducible representations, and therefore cannot be expressed in the tensor language of \citet{yun2020unifying}. Instead, we build on the optimization tools and comparable convergence assumptions of \citet{gunasekar2018linearconv} to explicitly characterize the stationary points of convergence for non-abelian G-CNNs.

We also note that our results are consistent with those of \citet{razinimplicit2020} showing that implicit generalization is often captured by measures of complexity which are quasi-norms, such as tensor rank. Our results prove that linear G-CNNs are biased towards low rank solutions in the Fourier regime, via regularization of the $2/L$-Schatten norms over Fourier matrix coefficients (also a quasi-norm). 
Lastly, there is a line of work focusing on understanding the expressivity~\citep{kondor2018generalization, cohen2019general_theory, yarotsky2021universal} and generalization~\citep{sannai2019improved, lyle2020benefits, bulusu2021generalization, elesedy2021provably} of equivariant  networks, but not specifically their implicit regularization.

To analyze bounded-width filters, which are more commonly used in practice, a recent work by \citet{jagadeesan2021bias} shows that the implicit regularization for an arbitrary filter width $K$ is unlikely to admit a closed-form solution. Separate from calculating the exact form of implicit regularization, there is a rich line of work that details the trade-offs between restricting 
a function in its real versus Fourier regimes via uncertainty principles \citep{meshulam1992uncertainty,wigderson2021uncertainty}. While the connection between uncertainty theorems and bounded-width convolutional neural networks has not been thoroughly explored,~\citet{caro2021local} and~\citet{nicola2021stability} highlight the importance of uncertainty principles for understanding the behaviour of modern CNNs.


\section{Notation}\label{sec:notation}
Throughout this text, we denote scalars in lowercase script ($a$), vectors in bold lowercase script ($\va$), matrices in either bold uppercase script ($\mA$) or lowercase script hat ($\widehat \va$) when vectors are transformed into the Fourier regime (see \autoref{def:group_fourier_transform}), and tensors in bold non-italic uppercase script ($\tA$). For $f$ a function with range in $\mathbb{C}$, we overload notation slightly and let $\overline f$ denote the function with an element-wise complex conjugate applied, i.e. $\overline f(x)=\overline{f(x)} $. If $f$ is defined on a group $G$, let $f^-(g)=f(g^{-1})$.
For a vector $\va \in \mathbb{C}^n$ or a matrix $\mA \in \mathbb{C}^{m \times n}$, we denote its conjugate transpose as $\va^\dagger$ and $\mA^\dagger$ respectively. 
We use $\langle\cdot, \cdot\rangle$ to denote the standard vector inner product between two vectors and $\langle\cdot,\cdot\rangle_M$ to denote the inner product between matrices defined as $\langle\mA,\mB\rangle_M=\tr(\mA \mB^\dagger)$. $\| \cdot \|_p$ denotes the vector $p$-norm ($p=2$ when subscript is hidden) and $\|\cdot\|^{(S)}_p$ denotes the $p$-Schatten norm\footnote{Despite our terminology, $p$-vector and $p$-Schatten norms are technically quasi-norms for $p<1$.} of a matrix (equivalent to the $p$-vector norm of the singular values of the matrix).

We denote groups by uppercase letters $G$, an irreducible representation (irrep) of a group $G$ by $\rho$ or $\rho_i$, and a complete set of irreps by $\widehat G$, so every unitary irrep $\rho$ is equivalent (up to isomorphism) to exactly one element of $\widehat{G}$. The dimension of a given irrep $\rho$ is $d_\rho$.

\section{Background in group theory and Group-Equivariant CNNs} \label{sec:grouptheorybackground}


In this study, we analyze linear G-CNNs in a binary classification setting, where hidden layers perform equivariant operations over a finite group $G$, and networks have no nonlinear activation function after their linear group operations. Inputs $\vx_i$ are vectors of dimension $|G|$ (\textit{i.e.,} vectorized group functions $\vx:G \to \mathbb{R}$), and targets $y_i$ are scalars taking the value of either $+1$ or $-1$. Hidden layers in our G-CNNs perform cross-correlation over a group $G$, defined as
\begin{equation}
    (\vg \star \vh)(u) = \sum_{v \in G} \vg(uv) \vh(v)
    \label{eq:cross-correlation}
\end{equation}
where $\vg, \vh:G \to \mathbb{R}$. Note that the above is equivariant to the left action of the group, \textit{i.e.,} if $w \in G$ and $\vg_w'(u) = \vg(wu)$, then $(\vg_w' \star \vh)(u) = (\vg \star \vh)(wu)$. The final layer of our G-CNN is a fully connected layer mapping vectors of length $|G|$ to scalars. We note that this final layer in general will not construct functions that are symmetric to the group operations, as strictly enforcing group invariance in this linear setting will result in trivial outputs (only scalings of the average value of the input). Nonetheless, this model still captures the composed convolutions of G-CNNs, and is similar in construction to many practical G-CNN models, whose earlier G-convolutions still capture useful high-level equivariant features. For instance, the spherical CNN of \citet{cohen2018spherical} also has a final fully connected layer.  


Analogous to the discrete Fourier transform, there exists a \emph{group} Fourier transform mapping a function into the Fourier basis over the irreps of $G$. 

\begin{figure}
    \centering
    \includegraphics[width=3.3in]{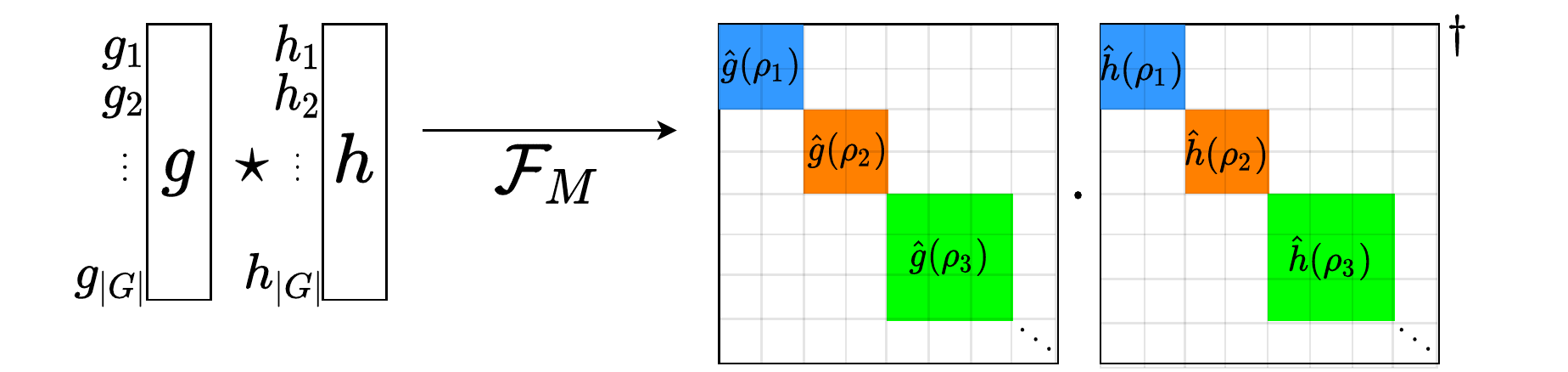}
    \caption{Cross-correlation of two functions over a group is equivalent to matrix multiplication over irreps (shown as blocks of a  larger matrix here) in Fourier space.}
    \label{fig:correlation}
\end{figure}

\begin{definition}[Group Fourier transform] \label{def:group_fourier_transform}
Let $f:G \to \mathbb{C}$. Given a fixed ordering of $G$, let $\ve_u$ be the standard basis vector in $\R^{|G|}$ that is $1$ at the location of $u$ and $0$ elsewhere. Then, $\vf = \sum_{u \in G} f(u) \ve_u$ is the vectorized function $\vf$. Given ${}\widehat{G}$ a complete set of unitary irreps of $G$, let $\rho \in {}\widehat{G}$ be a given irrep of dimension $d_\rho$, $\rho:G \longrightarrow \on{GL}\left(d_\rho,\,\mathbb{C}\right)$\footnote{Note that for an abelian group, $d_\rho = 1 \; \forall \rho$. For standard Fourier analysis over the cyclic group, each $\rho$ is a complex sinusoid at some frequency.}. The group Fourier transform of $f$, $\widehat{f}: \widehat{G} \rightarrow \mathbb{C}$ at a representation $\rho$ is defined as \citep{terras1999fourier} \begin{equation}
    \widehat{f}(\rho) = \sum_{u \in G} f(u) \rho(u).
\end{equation}
By choosing a fixed ordering of $\widehat{G}$, one can similarly construct $\widehat{\vf}$ as a block-diagonal matrix version of $\widehat{f}$ (as in \autoref{fig:correlation}).
We define $\cF_M$ to be the \emph{matrix} Fourier transform that takes $\vf$ to $\widehat{\vf}$: 
\begin{align}
\widehat \vf = \cF_M \vf  = \bigoplus_{\rho \in {}\widehat{G}} \widehat{f}(\rho)^{\oplus d_\rho} \;\in\on{GL}\left(|G|, \,\mathbb{C}\right).
\end{align}
$\widehat \vf$ or $\cF_M \vf$ 
are shortened notation for the complete Fourier transform. Furthermore, by vectorizing the matrix $\widehat{\vf}$, there is a \emph{unitary} matrix $\cF$ taking $\vf$ to $\widehat{\vf}$, 
analogous to the standard discrete Fourier matrix. We use the following explicit construction of $\cF$: denoting $\ve_{[\rho,i,j]}$ as the column-major vectorized basis for element $\rho_{ij}$ in the group Fourier transform, then we can form the matrix
\begin{equation}
    \cF = \sum_{u \in G} \sum_{\rho \in {}\widehat{G}} \frac{\sqrt{d_\rho}}{\sqrt{|G|}} \sum_{i,j=1}^{d_\rho} \rho(u)_{ij} \ve_{[\rho,i,j]} \ve_u^{T}.
\end{equation}
Intuitively, for each group element $g$, the matrix $\cF$ contains all the irrep images $\rho(g)$ `flattened' into a single column. See \autoref{app:gfourier} for further exposition.
\end{definition}

Convolution and cross-correlation are equivalent, up to scaling, to matrix operations after Fourier transformation. For example, for cross-correlation (\autoref{eq:cross-correlation}), $\widehat{(g \star h)} (\rho) = {}\widehat{g}(\rho) {}\widehat{h}(\rho)^\dagger$.
This simple fact, illustrated in \autoref{fig:correlation}, is behind the proofs of our implicit bias results.

\begin{figure*}[t!]
    \captionsetup[subfigure]{aboveskip=-1pt,belowskip=-3pt}  
  \begin{subfigure}{0.49 \textwidth}
    \includegraphics[width=\linewidth]{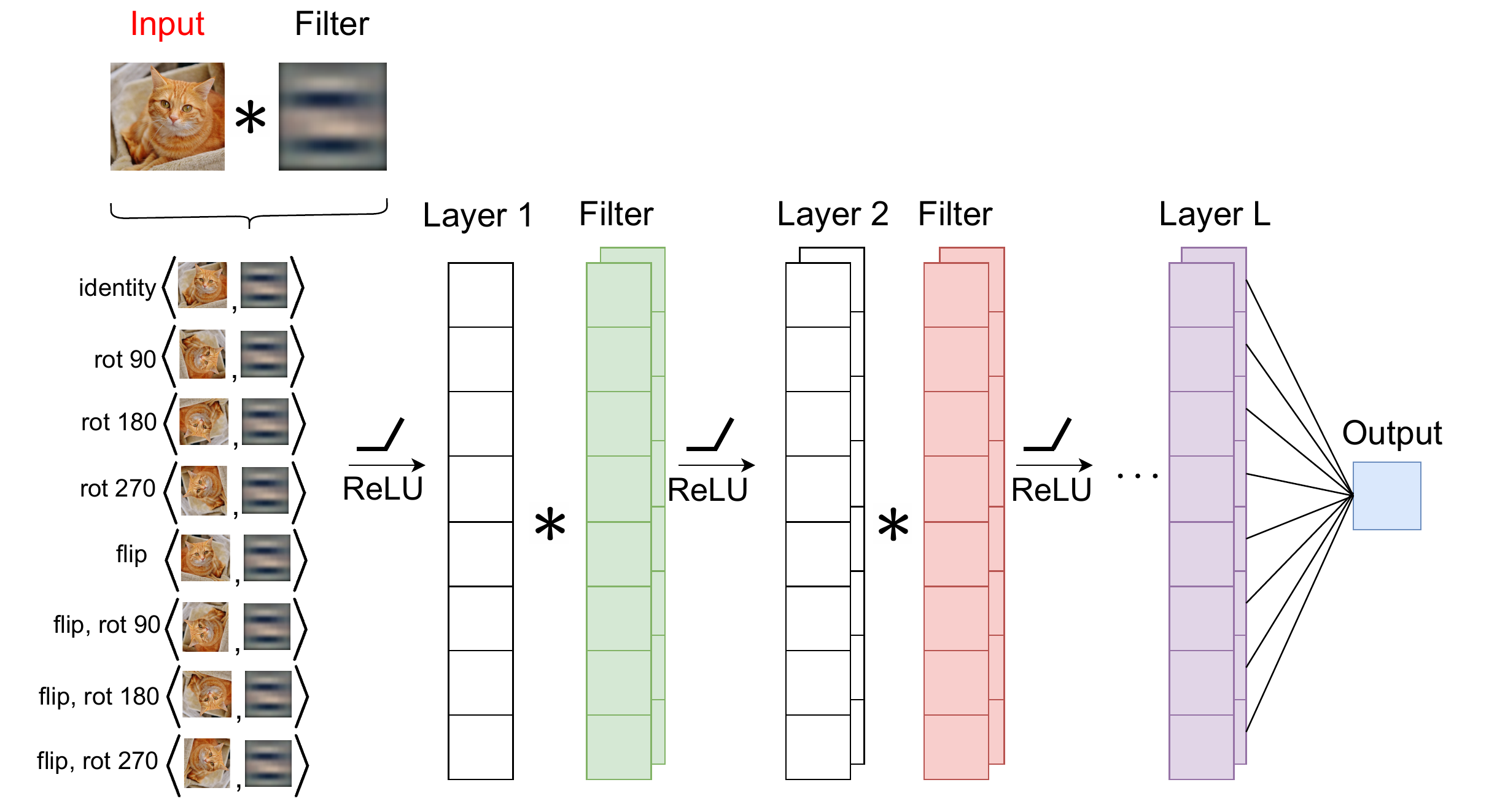}
    \caption{A practical G-CNN architecture for $D_8$.} \label{fig:arch_practice}
  \end{subfigure}%
  \vspace{1pt}
  \begin{subfigure}{0.49\textwidth}
    \includegraphics[width=\linewidth]{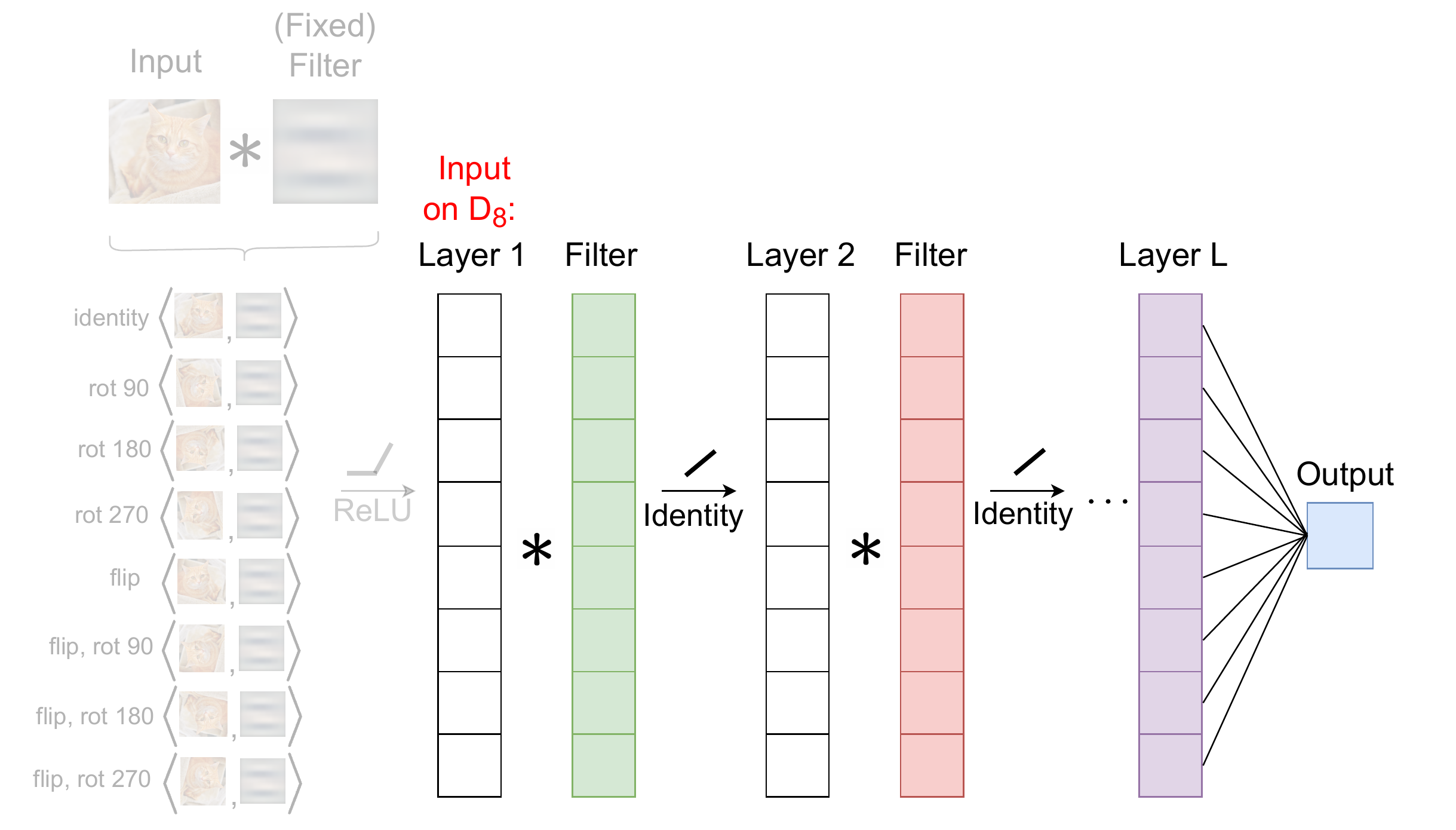}
    \caption{The linear G-CNN architecture we analyze for $D_8$.} \label{fig:arch_theory}
  \end{subfigure}%
  \hspace*{\fill}
\caption{A comparison between a practical G-CNN architecture for the group $D_8$, and its corresponding linear idealization that we will theoretically analyze. $D_8$ is a group of size 8, consisting of all rotations by 90 degrees and reflections. In a practical architecture, as shown in the first panel, the input may be an image, or anything upon which $D_8$ acts; it can be convolved over $D_8$ with respect to a learnable filter, yielding a function on $D_8$ (i.e. a function that takes $8$ values). The following layers intersperse $D_8$-convolutions, possibly with several channels, and nonlinearities, before a final fully connected layer yields a scalar output. In contrast, the second panel shows the simplified architecture we analyze here. In particular, there are no non-linearities and only a single channel is used. Furthermore, we assume the input is already a function on $D_8$, i.e. the input is an 8-dimensional vector. (One can think of this input as a fixed featurization of some image convolved over $D_8$ with a \emph{fixed} filter.)} 

\label{fig:architectures}
\end{figure*}

\section{Main Results} \label{sec:main_results}

We consider linear group-convolutional networks for classification analogous to those of \citet{yun2020unifying} and \citet{gunasekar2018linearconv}. A linear G-CNN is composed of several group cross-correlation layers followed by a fully connected layer. The input is formalized as a function on the group (according to some pre-defined ordering of elements), $\vx: G \rightarrow \R$, and the output is a scalar. Explicitly, let $G$ be a finite group with $\vx, \vw_1,\dots,\vw_L$ real-valued functions on $G$. The network output is $\on{NN}(\vx) = \langle \vx \star \vw_1 \star \cdots \star  \vw_{L-1} , \vw_L  \rangle \triangleq \langle \vx, \vbeta \rangle$. As an example, in Figure \ref{fig:architectures}, we illustrate both a ``practical" G-CNN architecture and this linearized version that we will study, for the group $D_8$. Let $\myW = [\vw_1\: \dots \: \vw_L]$ be the concatenation of all network parameters, and $\vbeta = \mathcal{P}(\myW) =  \vw_L \star \vw_{L-1}^- \star \cdots \star  \vw_1^-$ be the ``end-to-end" linear predictor consisting of composed cross-correlation operations. One can check (as we do in \autoref{lemma:net_fourier}) that $\widehat{\vbeta}\triangleq \cF_M \vbeta=\widehat{\vw_L}\dots \widehat{\vw_1}$. 
Networks are trained via gradient descent over the exponential loss function on linearly separable data $\{\vx_i,y_i\}_{i=1}^N$. Iterates take the form
\begin{align}
    \myW_\ell^{(t+1)} &= \myW_\ell^{(t)} - \eta_t \nabla_{\myW_\ell} \mathcal{L}(\mathcal{P}(\myW))
\end{align}
where $\ell( \langle \vx, \vbeta\rangle, y) = \exp(- \langle \vx, \vbeta \rangle \cdot y)$ and $ \mathcal{L}(\vbeta) = \sum_{i=1}^N \ell( \langle \vx_i, \vbeta\rangle, y_i)$. 

\subsection{Abelian}\label{subsec:abelian}
Similar to ordinary (cyclic) convolution\footnote{We include a canonical result in the appendix, \autoref{thm:fund_abelian_groups}, demonstrating that \emph{all} finite abelian groups are direct products of cyclic groups, i.e. \emph{multidimensional} translational symmetries.}, the commutative property of abelian groups implies that convolutions in real space are equivalently pointwise multiplication of irreps in Fourier space, since all irreps are one-dimensional for commutative groups ($d_\rho=1 \: \forall \rho$).
To start, recall the key definition of \citet{yun2020unifying}, determining which network architectures fall within the purview of their results:


\begin{proposition}[paraphrased from \cite{yun2020unifying}]
\label{prop:data_tensor_ortho_main}
Let $\tM(\vx)$ be a map from data $\vx \in \R^d$ to a data tensor $\tM(\vx) \in \R^{k_1 \times k_2 \times \cdots \times k_L}$. The input into an $L$-layer tensorized neural network can be written as an \emph{orthogonally decomposable data tensor} if there exists a full column rank matrix $\mS \in \C^{m \times d}$ and semi-unitary matrices $\mU_1, \dots, \mU_L \in \C^{k_\ell \times m}$ where $d \leq m \leq \min_\ell k_\ell$ such that:
\begin{align*}
    \tM(\vx) = \sum_{j=1}^m \left[ \mS\vx \right]_j \left( [\mU_1]_{\cdot, j} \otimes [\mU_2]_{\cdot, j} \otimes \cdots \otimes [\mU_L]_{\cdot, j}  \right)
\end{align*}
and moreover the network output is the tensor multiplication between $\tM(x)$ and each layer's parameters: 
\begin{align*}
   \on{NN}(\vx;\Theta) &= \tM(\vx) \cdot (\vw_1,\dots,\vw_L) \\ &= \sum_{i_1=1}^d \dots \sum_{i_L=1}^d \tM(x)_{i_1 \dots i_L} (\vw_1)_{i_1}\cdot \dots \cdot (\vw_L)_{i_L}
\end{align*}
\end{proposition}

Indeed, a linear G-CNN over an abelian group can be expressed in a way that satisfies \autoref{prop:data_tensor_ortho} for an appropriate choice of $\mS$ and $\mU_1,\dots,\mU_L$, as stated in the following proposition. 

\begin{proposition}\label{thm:tensconv}
Let $\tM(\vx)$ be an orthogonally decomposable data tensor with associated matrices $\mS,\mU_1,\dots,\mU_L$
as in \autoref{prop:data_tensor_ortho_main}. 
Given a finite abelian group $G$, let $d = m = k_\ell = \abs{G}$ and $\cF \in \C^{d \times d}$ be the group Fourier transform of $G$ (see \autoref{def:group_fourier_transform}).
With $\mS = d^{\frac{L-1}{2}} \cF$, unitary matrices $\mU_\ell = \cF^{-1} = \cF^\dagger$, and the data tensor $\tM(\vx)$ defined correspondingly, the output of a G-CNN with real-valued filters $\vw_1, \dots, \vw_L$ is a tensor operation: 
\begin{align*}
    \tM(\vx)\cdot (\vw_1,\dots \vw_L) = \langle \vx \star \vw_1 \star \cdots \star \vw_{L-1} , \vw_L\rangle
\end{align*}
\end{proposition}

The proof is deferred to \autoref{sec:proofs}. Fundamentally, the result requires not only that $\mathcal{F}$ is unitary, which holds for all finite groups, but also that cross-correlation is pointwise multiplication (up to a conjugate transpose) in Fourier space, i.e. $\mathcal{F}(\vx \star \vw) = (\mathcal{F}\vx) \odot (\mathcal{F}\vw)^\dagger $. This property only holds 
for commutative groups, as matrix multiplication is pointwise multiplication only for matrices of dimension $d_\rho = 1$. Given \autoref{thm:tensconv}, we  apply the implicit bias statement of \citet{yun2020unifying}. 

\begin{theorem}[Implicit regularization of linear G-CNNs for $G$ an abelian group]\label{thm:abelian}
Suppose there exists $\lambda > 0$ such that the initial directions $\bar{\vw}_1, \dots, \bar{\vw}_L$ of the network parameters satisfy $\abs{[\cF \bar{\vw}_\ell]_j }^2 - \abs{ [\cF \bar{\vw}_L]_j }^2 \geq \lambda $  for all $\ell \in [L-1]$ and $j \in [m]$, i.e. if the Fourier transform magnitudes of the initial directions look sufficiently different pointwise (which occurs with high probability for e.g. a Gaussian random initialization). Then, $\vbeta = \mathcal{P}([\vw_1,\dots,\vw_L])$ converges in a direction that aligns with a stationary point $\vz_\infty$ of the following optimization program:
\begin{equation}
    \min_{\vz \in \C^m}   \; \|\cF\vz\|_{2/L}  \; \; \text{ s.t. } \; \;  y_i \langle \vx_i, \vz \rangle \geq 1, \forall i \in [n]
\end{equation}
\end{theorem}

As noted in \autoref{thm:fund_abelian_groups} of the Appendix, all finite abelian groups can be expressed as a direct product of cyclic groups. In contrast, many groups (rotations, subgroups of permutations, etc.) with much richer structure are non-commutative, and we now turn our attention to the non-abelian case.

\subsection{Non-Abelian}
\label{subsec:nonabelian}
In Fourier space, non-abelian convolution consists of matrix multiplication over irreps, and does \emph{not} fit the pointwise multiplication structure of \autoref{thm:tensconv}. We instead build upon the results of \citet{gunasekar2018linearconv}, and directly analyze the stationary points of the proposed optimization program to prove the following: 
\begin{theorem}[Non-abelian; see also \autoref{thm:non-abelian}]\label{thm:non-abelian-informal}


Consider a classification task with ground-truth linear predictor $\vbeta$, trained via a linear G-CNN architecture with $L>2$ layers under the exponential loss. For almost all $\boldsymbol{\beta}$-separable datasets $\{\vx_i, y_i\}_{i=1}^n$, any bounded sequence of step sizes $\eta_t$, and almost all initializations: if the loss converges to 0, the gradients converge in direction, and the iterates themselves all converge in direction to a classifier with positive margin, then the resultant predictor is a scaling of a first order stationary point of the optimization problem:
\begin{equation} \label{eq:beta_opt_main}
    \min_{\vbeta} \| {}\widehat{\vbeta} \|^{(S)}_{2/L} \; \; \text{ s.t. } \; \; \forall n, y_n \left< \vx_n, \vbeta \right> \geq 1
\end{equation}
\end{theorem}
To prove the above statement, we 
show that linear G-CNNs converge to stationary points of \autoref{eq:beta_opt_main} via KKT conditions, which is also the high-level method of \citet{gunasekar2018linearconv}. However, our proof diverges in several key ways. 
First, we carefully redefine operations of the G-CNN as a series of inner products and cross-correlations with respect to the matrix Fourier transform of \autoref{def:group_fourier_transform}. Second, in this Fourier space, we analyze the subdifferential of the Schatten norms, to ultimately show that the KKT conditions of \autoref{eq:beta_opt_main} are satisfied. In contrast, \citet{gunasekar2018linearconv} analyze the subdifferential of a different objective, the ordinary $2/L$-vector norm. The fact that the irreps of a group are only unique up to isomorphism (\textit{e.g.,} conjugation by a unitary matrix) hints at the Schatten norm as the correct regularizer, since the Schatten norm is among the norms invariant to unitary matrix conjugation, but this must be confirmed by rigorous analysis. These features are specific to the non-abelian case. More specifically, the proof of this result follows the outline below:
\begin{enumerate}
    \item First, by applying a general result of \citet{gunasekar2018linearconv}, \autoref{thm:homog}, we can immediately characterize the implicit regularization in the full space of parameters, $\myW$ (in contrast to the end-to-end linear predictor $\vbeta$), as a (scaled) stationary point $\myW^\infty$ of the following optimization problem 
    in $\myW$:
    \begin{equation}\label{eq:w_opt_main}
        \min_{\myW \in \mathbb{R}^P} \|\myW\|_2^2 \; \; \; \text{s.t.} \;\;\; \forall n, y_n \langle \vx_n, \mathcal{P}(\myW) \rangle \geq 1
    \end{equation}
    \item Separately, we define a \emph{distinct} optimization problem, \autoref{eq:beta_opt_main}, in $\vbeta$, with the aim of showing that stationary points of \autoref{eq:w_opt_main} are a subset of those of \autoref{eq:beta_opt_main}, up to scaling.
    \item The \emph{necessary} KKT conditions for \autoref{eq:w_opt_main} characterize its stationary points: 
\begin{equation} \label{eq:kkt_w_main}
\begin{split}
    & \exists \{\alpha_n \geq 0\}
    : \alpha_n = 0 \; \text{if} \; y_n \langle \vx_n, \mathcal{P}(\myW^\infty)\rangle > 1  \\
    & \vw_i^\infty 
    = \nabla_{\vw_i} \Big\langle \mathcal{P}(\myW^\infty), \sum_n \alpha_n y_n \vx_n \Big\rangle
\end{split}
\end{equation}
    
    From here, we show that the \emph{sufficient} KKT conditions for \autoref{eq:beta_opt_main} are also satisfied by the corresponding end-to-end predictor. In particular, we calculate the set of subgradients\footnote{$\partial^o$ is the local subgradient of \cite{clarke1975generalized}: $\partial^o f(\vbeta) = \text{conv}\{ \lim_{i \to \infty} \nabla f(\vbeta+\vh_i): \vh_i \to 0\} $} $ \partial^o \| \widehat{\vbeta} \|_p^{(S)}$  for $p=\frac{2}{L}<1$, and then use \autoref{eq:kkt_w_main} to derive recurrences demonstrating that a positive scaling of $\sum_n \alpha_n y_n \widehat{\vx}_n$ is a member of this set. 
\end{enumerate}

\begin{remark}
For abelian groups where all irreps are one-dimensional, $\widehat{\vbeta}$ in \autoref{thm:non-abelian-informal} is a diagonal matrix. Thus, the $p$-Schatten norm coincides with the $p$-vector norm of the diagonal entries, recovering results in \autoref{subsec:nonabelian}. However, \autoref{thm:non-abelian-informal} requires stronger convergence assumptions. 
\end{remark}

\textbf{Infinite dimensional groups: }\autoref{thm:non-abelian-informal} applies to all \emph{finite} groups, but G-CNNs have extensive applications for \emph{infinite} groups, where outputs of convolutions are infinite-dimensional. Here, it is common to assume sparsity in the Fourier coefficients and ``band-limit" 
filters over a set of low-frequency irreps (under some natural group-specific ordering) that form a finite dimensional linear subspace 
(we denote the representation of a function in this band-limited Fourier space as $\underline{\widehat{\vw}}$). G-CNNs with band-limited filters take precisely the form of the finite G-CNNs from \autoref{thm:non-abelian-informal}. 
Thus, slight modifications yield the following for infinite groups (see Appendix \ref{subsubsec:infinite} for details). 

\begin{corollary}[Infinite-dimensional groups with band-limited functions; see also \autoref{thm:infinite-formal}]
Let $G$ be a compact Lie group with irreps $\widehat G$, and let $B \subset \widehat G$ with $|B| < \infty$.\footnote{For example, $G=SO(3)$ and $B$ indexes all Wigner d-matrices with $|\ell| \leq L$ \citep{kondor2018clebsch}.} Proceed fully in Fourier space, in the subspace corresponding to $B$: consider a linearly separable classification task with ground-truth linear predictor $\underline{\widehat{\vbeta}}$ and inputs $\underline{\widehat\vx}$, both real-valued and supported only on irreps in $B$, and proceed by gradient descent on the band-limited Fourier-space filters. Under near-identical conditions as \autoref{thm:non-abelian-informal}, the resultant predictor is a scaling of a first order stationary point of:
\begin{equation} 
    \min_{\underline{\widehat{\vbeta}}} \Big\| \underline{\widehat{\vbeta}} \Big\|^{(S)}_{2/L} \; \; \text{ s.t. } \; \; \forall n, y_n \left< \underline{\widehat{\vx}}_n, \underline{\widehat {\vbeta}} \right> \geq 1
\end{equation}
\end{corollary}

\section{Experiments}
\label{sec:experiments}

We first experimentally confirm our theory in a simple setting illustrating the effects of implicit regularization.\footnote{Our code is available here: \url{https://github.com/kristian-georgiev/implicit-bias-of-linear-equivariant-networks}} Then, we relax the crucial assumption of linearity in our setup, to empirically show that the results may hold locally even in nonlinear settings (including the practical case of spherical CNNs \citep{cohen2018spherical}). Note that the results for nonlinear networks in \autoref{subs:weregettinwild} are \textit{only} empirical in nature, and Theorems~\ref{thm:abelian} and~\ref{thm:non-abelian} do not necessarily hold in the more general nonlinear setting. For all binary classification tasks, we use three-layer networks with inputs and convolution weights in $\mathbb{R}^{|G|}$, and all plots begin at epoch 1. 
Since we are interested in the resulting implicit bias alone, we only analyze loss on data in the training set. 
A complete description of our experimental setup can be found in \autoref{appa:comp}. 

Throughout this section, we plot norms in the Fourier and real regimes side-by-side to highlight unavoidable trade-offs in implicit regularization between the two conjugate regimes. These trade-offs, which have a rich history of study in physics and group theory, are commonly termed uncertainty principles \citep{wigderson2021uncertainty}. Since non-abelian groups have matrix-valued irreducible representations, uncertainty theorems must account for norms and notions of support in the context of matrices. One especially relevant uncertainty theorem states that sparseness in the real or Fourier regime necessarily implies dense support in the conjugate regime:

\begin{theorem}[Meshulam uncertainty theorem \citep{meshulam1992uncertainty}]
Given a finite group $G$ and $f:G \to \mathbb{C}$, let ${}\widehat{G}$ be the set of irreps of $G$ and $\vf$ be the vectorized function (see \autoref{def:group_fourier_transform}). Then
\begin{equation}
     |\on{supp}(\vf)|\; \left[ \sum_{\rho \in {}\widehat{G}} d_{\rho} \on{rank}\left({}\widehat{\vf}(\rho) \right) \right] \geq |G|
\end{equation}
\end{theorem}
The theorem above shows that the rank of a function's Fourier matrix coefficients is the proper notion of support in the uncertainty theorem for a non-abelian group. In the context of our result, this implies that the learned linear function is likely 
to have large support in real space (at least locally, with respect to the network parameters).
Other uncertainty principles are detailed in \autoref{app:uncertainty}. 

\begin{remark}[Generalization]
The implication of our results on generalization are highly task and data-dependent. Indeed, one could create \emph{contrived} experiments where the inductive bias leads to \emph{worse} or \emph{better} results, which would depend entirely on whether the ground-truth linear predictor has small Fourier Schatten norm (i.e. based on the uncertainty principle above, whether the ground-truth function is sparse in real or Fourier space). Noting that band-limited functions have small Schatten norm, and that practical spherical CNNs band-limit (as natural spherical images can often be well-represented by their first few spherical harmonics), there is reason to believe this implicit bias could aid in generalization on natural data distributions, but the primary intention of our result is purely to highlight that such an implicit bias exists and characterize it mathematically. As such, our experiments demonstrate only effects on training error, rather than generalization. 

\end{remark}

\subsection{Empirical Confirmation of Theory}
\label{subsec:experiments_confirm}

\begin{figure*}[h!]
 \captionsetup[subfigure]{aboveskip=-1pt,belowskip=-3pt}  
 \begin{subfigure}{0.49\textwidth}
    \includegraphics[width=\linewidth]{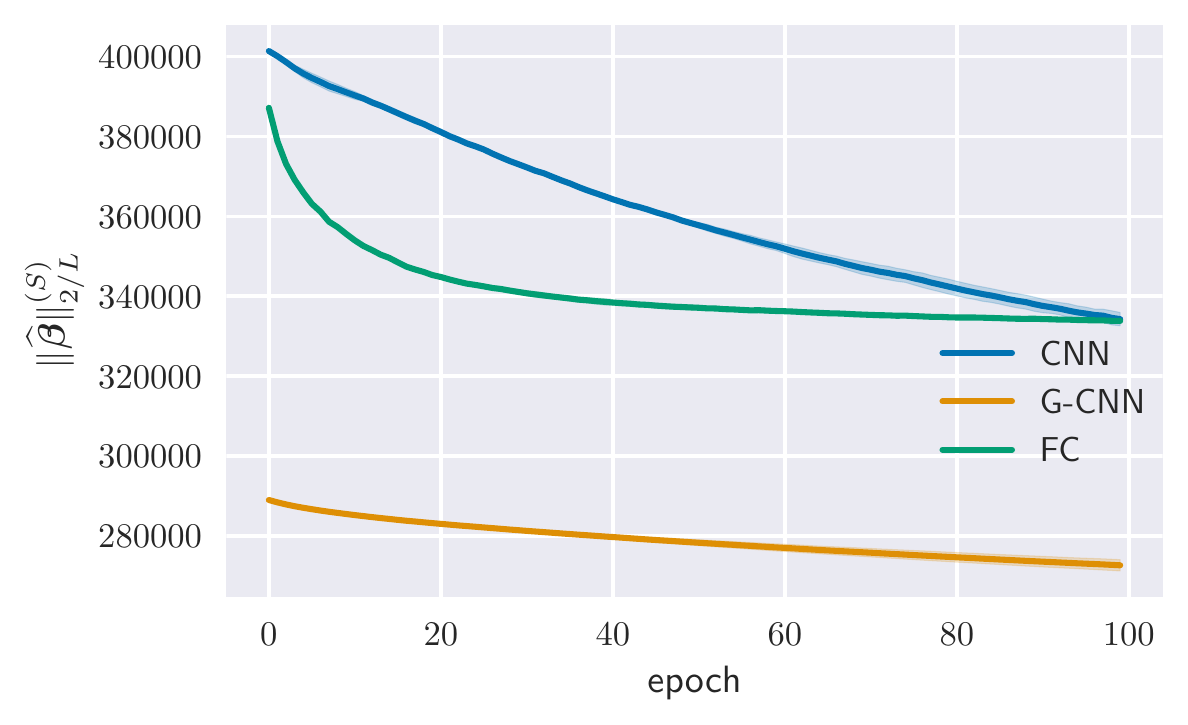}
    \caption{Fourier space norm of network linearization $\vbeta$} \label{fig:MNIST_linear_sub_a}
  \end{subfigure}%
  \hspace*{\fill}   
  \begin{subfigure}{0.49\textwidth}
    \includegraphics[width=\linewidth]{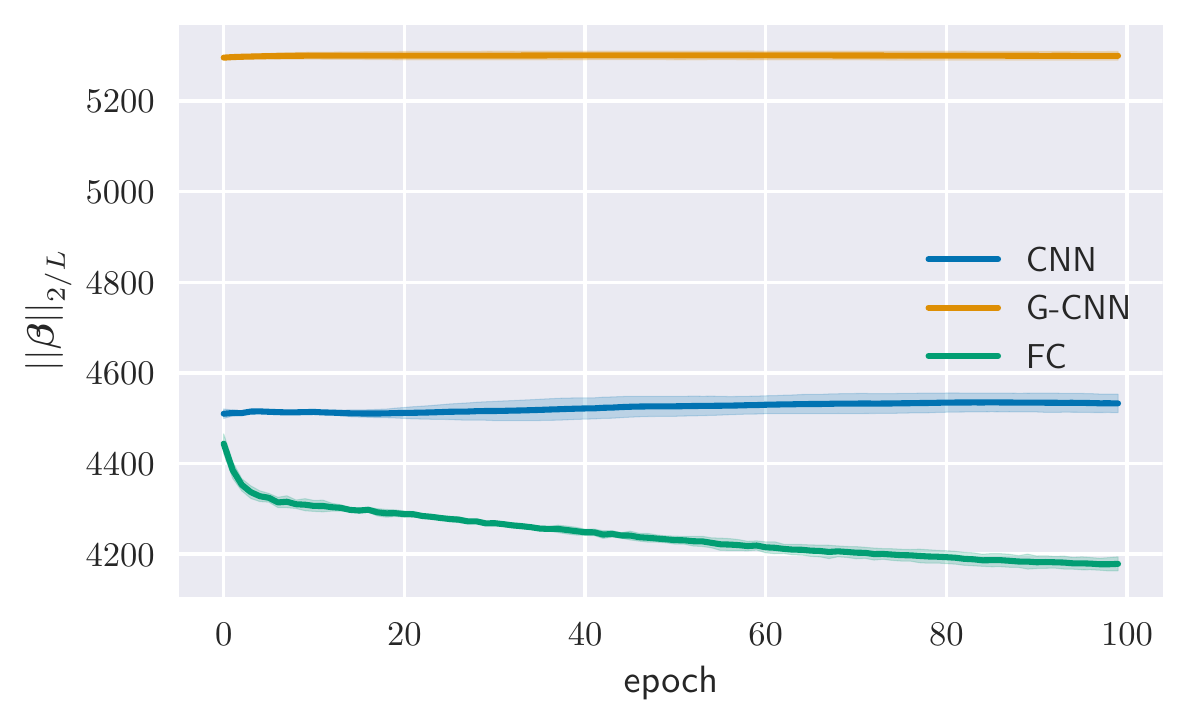}
    \caption{Real space norm of network linearization $\vbeta$} \label{fig:MNIST_linear_sub_b}
  \end{subfigure}%
  \hspace*{\fill}
\caption{Norms of the linearizations of three different linear architectures for the non-abelian group ${G=\left(C_{28}\times C_{28}\right) \times D_8}$ trained using the digits $1$ and $5$ from the MNIST dataset.} \label{fig:MNIST_linear}
\end{figure*}

We first  trace the regularization through (training) epochs for networks trained to classify data with $\pm 1$ labels. We consider three groups here, with more in \autoref{appa:comp}. \autoref{fig:d8_gaussian_2} shows the implicit bias for a G-CNN over the dihedral group $D_8$, a simple non-abelian group that captures the geometry of a square\footnote{$D_n$ denotes the dihedral group of order $n$.}. Inputs are vectors with elements drawn i.i.d. from the standard normal distribution. \autoref{fig:MNIST_linear} shows the implicit bias for a G-CNN over the non-abelian group $\left(C_{28}\times C_{28}\right) \times D_8$ which acts on images (the digits 1 and 5) from the MNIST dataset. 

In the three settings above, we compare the behaviors of G-CNN, traditional CNN\footnote{``CNN" generically refers to a G-CNN over the cyclic group of size equal to the size of the input.}, and fully-connected (FC) network architectures with similar instantiations. We plot the real space vector norm and Fourier space Schatten norm of the network linearization over training epochs. All models perfectly fit the data in this overparameterized setting, and convergence to a given regularized solution corresponds to convergence in the loss to zero.

\begin{figure*}[h!]
    \centering
    \includegraphics[]{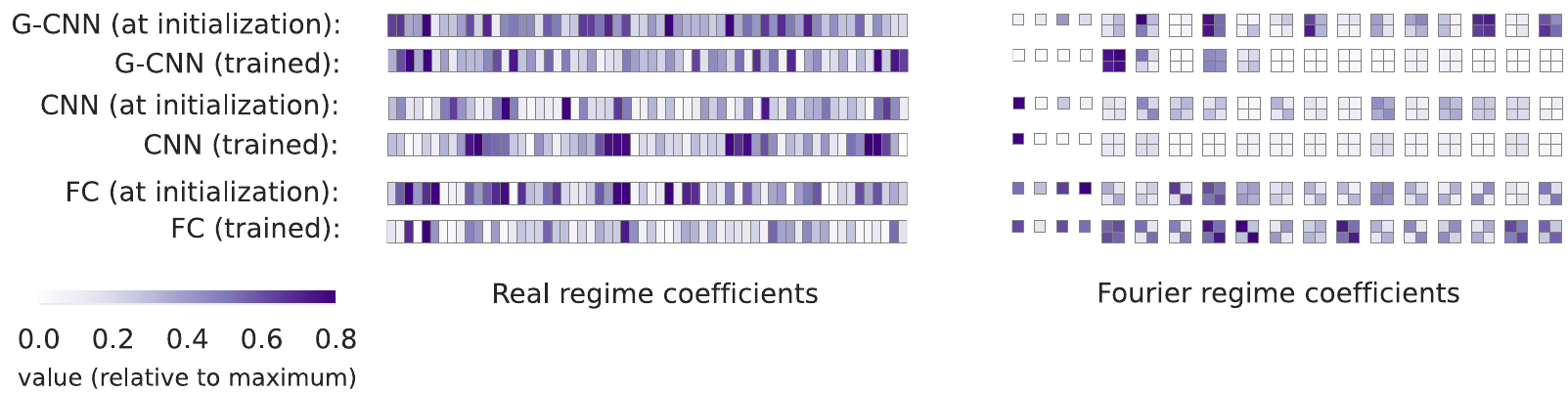} 
    \caption{Linearized functions of the G-CNN (over the dihedral group $D_{60}$) are sparse in the Fourier regime of the group. Furthermore, the sparsity pattern shows up in blocks of $4$, which are the blocks containing coefficients of individual irreps of the dihedral group $D_{60}$. Linearized functions in the real regime of the G-CNN, in contrast, are rather dense, highlighting the uncertainty principles inherent in the implicit bias of G-CNNs. Sparsity in the group Fourier regime is not evident in CNNs or fully connected (FC) networks. This is the same setting as that in \autoref{fig:d8_gaussian_2}, except with $D_{60}$ instead of $D_8$ (see \autoref{app:visualizing_bias} for more details). }
    \label{fig:d60_viz_linearization}
\end{figure*}

Consistent with theory, G-CNN architectures shown in Figures \ref{fig:d8_gaussian_2} and \ref{fig:MNIST_linear} have the smallest Fourier space Schatten norms among the architectures. Note that since our theory only implies a G-CNN will reach a stationary point of the Schatten norm minimization \emph{subject to fitting the training data}, Schatten norms greater than $0$ are expected. The form of the implicit bias is visualized in the example shown in \autoref{fig:d60_viz_linearization}, which shows the values of the linearization over irreps in Fourier space (see \autoref{app:visualizing_bias} for further details). As expected, the G-CNN outputs a linearization that is sparse over low-rank irreps in the Fourier regime. FC networks exhibit no group Fourier regime regularization, while standard CNNs exhibit some regularization since their irreducible representations are similar to those of the $D_{60}$-CNN (see e.g. \ref{eq:dihedral_irreps} in the Appendix). The differing behaviors of CNNs and FC networks show that implicit regularization is a consequence of the choice of architecture, and not inherent to the task itself.



\subsection{Assessment of Theory on Nonlinear Networks}
\label{subs:weregettinwild}

\begin{figure*}[t!]
    \captionsetup[subfigure]{aboveskip=-1pt,belowskip=-3pt}  
    \begin{subfigure}{0.49 \textwidth}
    \includegraphics[width=\linewidth]{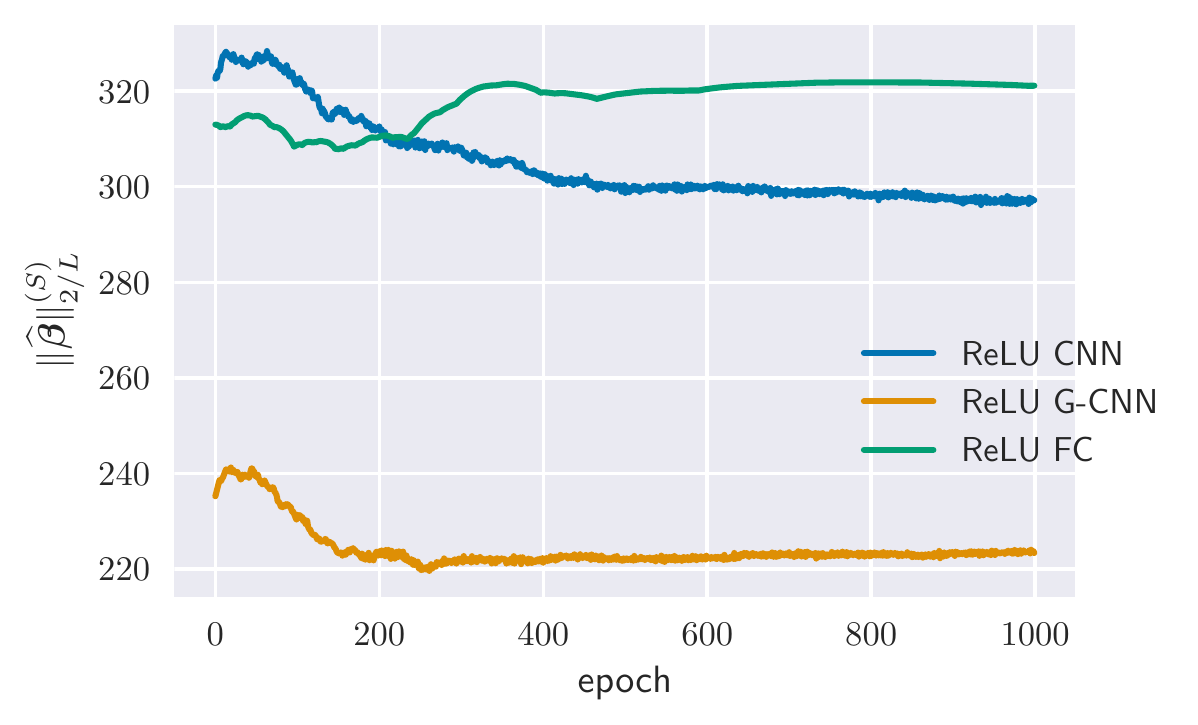}
    \caption{A $G$-CNN on non-abelian  $G=D_{60}$.} \label{fig:relu_figb_d60}
  \end{subfigure}%
  \hspace*{\fill}   
   \begin{subfigure}{0.49 \textwidth}
    \includegraphics[width=\linewidth]{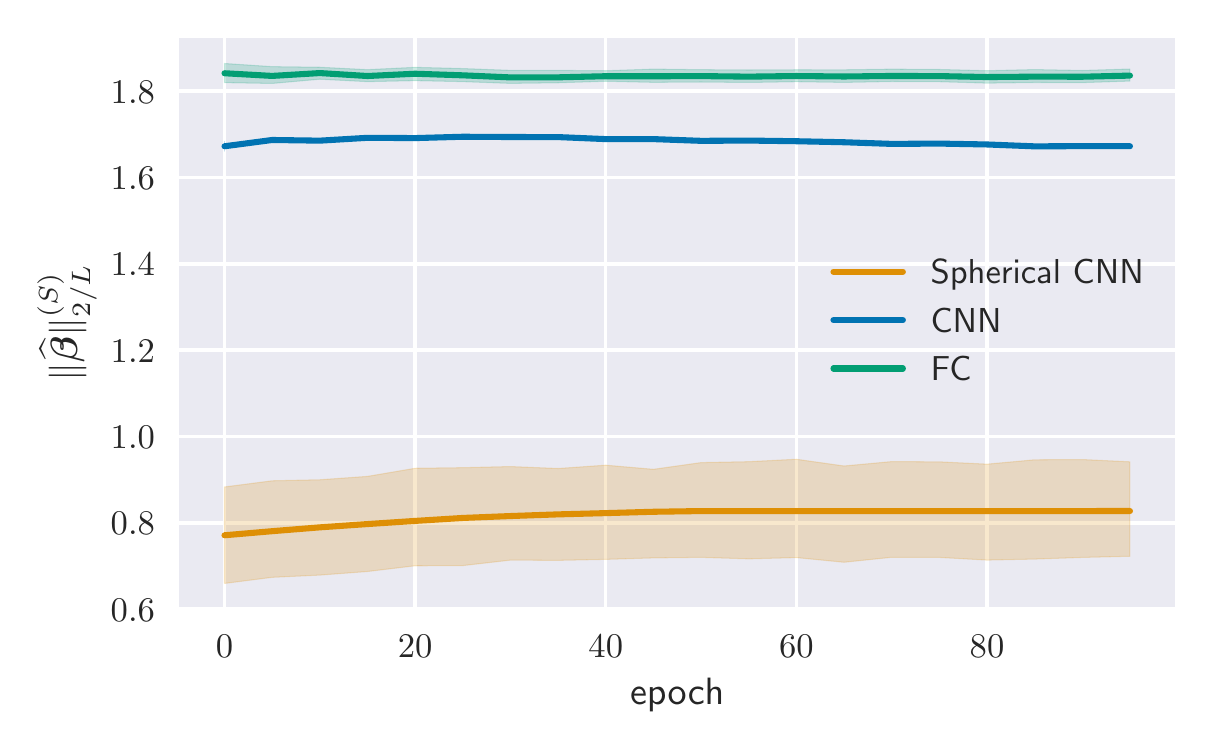}
    \caption{A Spherical CNN on bandlimited $G=SO(3)$. } \label{fig:relu_scnn}
  \end{subfigure}%
  \hspace*{\fill}
\caption{Group Fourier norms for \textit{nonlinear} architectures with ReLU activations show that nonlinear G-CNNs implicitly regularize locally. Both figures track the mean norm of the per-sample local linearizations of each network. 
Figure~\ref{fig:relu_figb_d60} is on a network with a final linear layer, and performs binary classification on 10 isotropic Gaussian data points. Figure~\ref{fig:relu_scnn} is on a spherical CNN, trained on rotated spherical MNIST  \citep{cohen2018spherical}. See \cref{app:spherical_details} for details. 
} \label{fig:relu} \end{figure*}

Here we introduce rectified linear unit (ReLU) nonlinearities between hidden layers to analyze implicit bias in the presence of nonlinearity. Our theoretical results do not necessarily hold in this case, so we are \textit{exploring, rather than confirming,} their validity for nonlinear networks.
Given a G-CNN with ReLU activations, we wish to calculate the Schatten norm of the Fourier matrix coefficients of the network linearization $\boldsymbol{\beta}$. 
However, networks can no longer be collapsed into linear functions due to the nonlinearities. Instead, we construct local linear approximations at each data point via a first-order Taylor expansion, calculate the norms of interest according to this linearization, and average the results across the dataset to get a single aggregate value. We evaluate the implicit bias of 
a nonlinear G-CNN (with linear final layer) on the dihedral group $D_{60}$ with synthetic data. We also evaluate a nonlinear spherical-CNN, where $G=SO(3)$ is an infinite group, with spherical MNIST data; see \citet{cohen2018spherical} for details. 
Linear approximations are used only to analyze implicit regularization, and not to further interpret the outputs of the locally linear neural network, as such analysis can give rise to misleading or fragile feature importance maps \citep{ghorbani2019interpretation}.


Remarkably, as shown in Figures~\ref{fig:relu_figb_d60} and~\ref{fig:relu_scnn}, our results remain valid in this nonlinear setting. While this does not guarantee that our implicit bias characterization will hold in more general settings, it is encouraging that our theoretical predictions seem to numerically hold, despite the violation of linearity and, in the case of Figure~\ref{fig:relu_scnn}, the cross-entropy loss function. Additional figures detailing the real-space behaviour are provided in \autoref{appa:comp}.

\section{Discussion}\label{sec:discussion}
In this work, we have shown that $L$-layer linear G-CNNs with full width kernels are biased towards sparse solutions in the Fourier regime regularized by the $2/L$-Schatten norm over Fourier matrix coefficients. Our analysis applies to linear G-CNNs, over either finite groups or infinite groups with band-limited inputs, which are trained to perform binary classification. In advancing our results on implicit regularization, we highlight some limitations of this work and important future directions:
\begin{itemize}
    \itemsep0em 
    \item \textbf{Nonlinearities: }Adding nonlinearities to the G-CNNs studied here expands the space of functions which the G-CNNs can express, but implicit regularization in this nonlinear setting may be challenging to characterize as G-CNNs are no longer linear predictors. Local analysis may be possible in special cases, \textit{e.g.,} in the infinite width limit \citep{lee2017deep}.
    \item \textbf{Bounded width kernels: }Our results apply to full-width kernels, supported on the entire group. Expanding results to bounded-width kernels, \textit{i.e.,} those with sparse support, is an obvious future direction, though prior work indicates that closed form solutions may not exist for these  special cases \citep{jagadeesan2021bias}.
    \item \textbf{Different loss function and learning settings: } We study the exponential loss function on binary classification. It is an open question how the implicit bias changes for classification over more than two classes, even for CNNs and fully-connected networks, as well as G-CNNs~\citep{gunasekar2018linearconv}. 
\end{itemize}


Although concise implicit regularization measures are challenging to analyze for realistic, nonlinear architectures, linear networks provide an instructive case study with precise analytic results. In this work, by proving that linear group equivariant CNNs trained with gradient descent are regularized towards low-rank matrices in Fourier space, we hope to advance the broader agenda of understanding generalization, and in particular how networks with diverse architectures --- particularly those with built-in symmetries --- learn. 

\subsubsection*{Acknowledgments}
We thank Ankur Moitra and Chulhee Yun for helpful discussions early in the development of this project. HL is supported by the Fannie and John Hertz Foundation and the National Science Foundation Graduate Research Fellowship under Grant No. 1745302. 

\bibliographystyle{abbrvnat}
\bibliography{main}

\newpage
\appendix
\onecolumn
\numberwithin{equation}{section}

\section{Proofs}\label{sec:proofs}
We begin with a simple lemma detailing the Fourier space end-to-end predictor for a G-CNN.

\begin{lemma}[G-CNN in Fourier space]\label{lemma:net_fourier}
A G-CNN given by $\langle \vx \star \vw_1 \star \cdots \star \vw_{L-1}, \vw_L \rangle$ is equivalent to $\langle \vx, \vw_L \star \vw_{L-1}^- \star \dots \star \vw_1^- \rangle$, or in Fourier space to $\frac{1}{|G|} \langle \widehat{\vx}, \widehat{\vw}_L \cdots \widehat{\vw}_1 \rangle_M$.
\end{lemma}

\begin{proof}
\begin{equation}
    \begin{split}
        N(x) &= \langle \vx \star \vw_1 \star \cdots \star \vw_{L-1}, \vw_L \rangle \\
        & = \frac{1}{|G|} \langle \cF_M (\vx \star \vw_1 \star \cdots \star \vw_{L-1}), \cF_M \vw_L \rangle_M \\
        & = \frac{1}{|G|} \tr[ \cF_M (\vx \star \vw_1 \star \cdots \star \vw_{L-1}) \widehat{\vw}_L^\dagger] \\
        & = \frac{1}{|G|} \tr[ \cF_M (\vx \star \vw_1 \star \cdots \star \vw_{L-2}) \widehat{\vw}_{L-1}^\dagger \widehat{\vw}_L^\dagger] \\
        & = \frac{1}{|G|} \tr[ \widehat{\vx} \widehat{\vw}_{1}^\dagger \cdots \widehat{\vw}_{L-1}^\dagger \widehat{\vw}_L^\dagger] \\
        & = \frac{1}{|G|} \langle \widehat{\vx}, \widehat{\vw}_L \cdots \widehat{\vw}_1 \rangle_M
    \end{split}
\end{equation}
To see the stated equality in real space, observe that $\widehat{\vw_1^-} = \widehat{\vw_1}^\dagger$ by definition. Thus, $$\cF_M (\vw_L \star \vw_{L-1}^- \star \dots \star \vw_1^-) =  \widehat{\vw}_L \cdots \widehat{\vw}_1$$ 
\end{proof}

\subsection{Abelian}
The proof of \autoref{thm:tensconv} is included below.

First, recall the primary theorem of \citet{yun2020unifying}:
\begin{theorem}[paraphrased from \cite{yun2020unifying}] \label{thm:yun} 
If there exists $\lambda > 0$ such that the initial directions $\bar{\vv}_1, \dots, \bar{\vv}_L$ of the network parameters satisfy $\abs{[\cF \bar{\vv}_\ell]_j }^2 - \abs{ [\cF \bar{\vv}_L]_j }^2 \geq \lambda $  for all $\ell \in [L-1]$ and $j \in [m]$, i.e. of the Fourier transform magnitudes of the initial directions look sufficiently different pointwise (which is likely for e.g. a random initialization), then $\vbeta(\Theta(t))$ converges in a direction that aligns with $\mS^T \vrho_\infty$ where $\vrho_\infty \in \C^m$ denotes a stationary point of the following optimization program:
\begin{equation}
    \min_{\vrho \in \C^m} \; \; \|\vrho\|_{2/L} \; \; \text{ s.t. } \; \; y_i\vx_i^T\cF^Tv\rho \geq 1, \forall i \in [n]
\end{equation}
Since $\mS=\cF$ is invertible, then in fact $\beta(\Theta(t))$ converges in a direction that aligns with a stationary point $\vz_\infty$ of the following optimization program:
\begin{equation}
    \min_{\vz \in \C^m} \; \; \|\cF \vz\|_{2/L} \; \; \text{ s.t. } \; \; y_i \vx_i^T \vz \geq 1, \forall i \in [n]
\end{equation}
\end{theorem}

We proceed by showing that abelian G-CNNs can be written as a vector of parameters contracted (according to tensor operations) with an orthogonally decomposable data tensor, which is the primary condition for \autoref{thm:yun} to hold.


\begin{proposition}[paraphrased from \cite{yun2020unifying}]
\label{prop:data_tensor_ortho}
Let $\tM(\vx)$ be a function that maps data $\vx \in \R^d$ to a data tensor $\tM(\vx) \in \R^{k_1 \times k_2 \times \cdots \times k_L}$. The data input into an $L$-layer tensorized neural network can be written in the form of an \textbf{orthogonally decomposable data tensor} if there exists a full column rank matrix $\mS \in \C^{m \times d}$ and semi-unitary matrices $\mU_1, \dots, \mU_L \in \C^{k_\ell \times m}$ where $d \leq m \leq \min_\ell k_\ell$ such that
$\tM(\vx)$ can be written as:
\begin{equation}
    \tM(\vx) = \sum_{j=1}^m \left[ \mS\vx \right]_j \left( [\mU_1]_{\cdot, j} \otimes [\mU_2]_{\cdot, j} \otimes \cdots \otimes [\mU_L]_{\cdot, j}  \right)
\end{equation}
and such that the network output is the tensor multiplication between $\tM(x)$ and each layer's parameters: 
\begin{align*}
   \on{NN}(\vx;\Theta) &= \tM(\vx) \cdot (\myW_1,\dots,\myW_L) \\
    &= \sum_{i_1=1}^d \dots \sum_{i_L=1}^d \tM(\vx)_{i_1 \dots i_L} (\myW_1)_{i_1}\cdot \dots \cdot (\myW_L)_{i_L}
\end{align*}
\end{proposition}

\begin{proof}[For \autoref{prop:data_tensor_ortho}]\label{pf:abelian_tensor_form}
By direct manipulation:
\begin{align*}
    f(\vx;\Theta) &= \tM(\vx) \cdot (\vv_1,\dots,\vv_L) \\
    &= \sum_{i_1=1}^d \dots \sum_{i_L=1}^d \tM(\vx)_{i_1 \dots i_L} (\vv_1)_{i_1}\cdot \dots \cdot (\vv_L)_{i_L} \\
    &= \sum_{i_1=1}^d \dots \sum_{i_L=1}^d \Bigg( \sum_{j=1}^d \left[ \mS \vx \right]_j \left( [\mU_1]_{\cdot, j} \otimes [\mU_2]_{\cdot, j} \otimes \cdots \otimes [\mU_L]_{\cdot, j}  \right) \Bigg)_{i_1\dots i_L} (\vv_1)_{i_1}\cdot \dots \cdot (\vv_L)_{i_L}  \\
    &= \sum_{i_1=1}^d \dots \sum_{i_L=1}^d \Bigg( \sum_{j=1}^d \left[ \mS\vx \right]_j \left( [\mU_1]_{i_1, j}  [\mU_2]_{i_2, j}  \cdots  [\mU_L]_{i_L, j}  \right) \Bigg) (\vv_1)_{i_1}\cdot \dots \cdot (\vv_L)_{i_L} \\
    &=  \sum_{j=1}^d \left[ \mS\vx \right]_j \left( [\mU_1^T\vv_1]_{i_1}  [\mU_2^T\vv_2]_{i_2}  \cdots [\mU_L^T\vv_L]_{i_L}  \right)   \\
    &= d^{\frac{L-1}{2}} \sum_{j=1}^d \left[ \cF\vx \right]_j \left( [\overline{\cF}\vv_1]_{i_1}  [\overline{\cF}\vv_2]_{i_2}  \cdots [\overline{\cF}\vv_L]_{i_L}  \right)   \\
    &= d^{\frac{L-1}{2}} \sum_{j=1}^d \left[ \cF\vx \right]_j \left( [\overline{\cF \vv_1}]_{i_1}  [\overline{\cF \vv_2}]_{i_2}  \cdots [\overline{\cF \vv_L}]_{i_L}  \right)   \\
    &= d^{\frac{L-1}{2}} \sum_{j=1}^d \left[ \cF\vx \right]_j \overline{\left( [\cF \vv_1]_{i_1}  [\cF \vv_2]_{i_2}  \cdots [\cF \vv_L]_{i_L}  \right) }  \\
    &= \langle \cF \vx, \cF\vv_1 \odot \dots \odot \cF  \vv_L \rangle \\
    &= \langle \cF \vx \odot \overline{\cF\vv_1}, \cF\vv_2 \odot \dots \odot \cF  \vv_L \rangle \\
    &= \langle \cF (\vx \star \vv_1) \odot \overline{\cF\vv_2}, \cF\vv_3 \odot \dots \odot \cF  \vv_L \rangle \\
    &= \langle \cF (\vx \star \vv_1 \star \vv_2) \odot \overline{\cF\vv_3}, \cF\vv_4 \odot \dots \odot \cF  \vv_L \rangle \\
    &= \langle \cF (\vx \star \vv_1 \star \vv_2 \star \cdots \star \vv_{L-1}) ,  \cF  \vv_L \rangle \\
    &= \langle \vx \star \vv_1 \star \dots \star \vv_{L-1}, \vv_L \rangle \\
\end{align*}
Here, we have used that the filters are real-valued.
\end{proof}
Note that \autoref{thm:abelian} then merely requires that $||\cF^{-T} \vz || = || \overline{\cF}\vz|| = ||\cF \vz||$, for real-valued $\vz$.

\subsubsection{Fundamental theorem of finite abelian groups}
While the proof in the previous section is complete and correct, intuition (and/or alternate analysis) for abelian groups is aided by the important fact that all finite abelian groups are a direct product of cyclic groups.

\begin{theorem}[Fundamental theorem of finite abelian groups \citep{dummit2004abstract}]\label{thm:fund_abelian_groups}
Any finite abelian group is a direct product of a finite number of cyclic groups whose orders are prime powers uniquely determined by the group.
\end{theorem}

Given a decomposition of an abelian group $G$ into $k$ cyclic groups $C_{d_1} \times \cdots \times C_{d_k}$, one can easily construct the group Fourier transform as a Kronecker product of discrete Fourier transform matrices which are the group Fourier transforms of the respective cyclic groups.
\begin{equation}
    G = C_{d_1} \times \cdots \times C_{d_k} \; \implies \; \cF = \bigotimes_{i=1}^k \cF_{d_i},
    \label{eq:abelian_fourier_transform}
\end{equation}
where $\bigotimes$ denotes the Kronecker product over matrices and $\cF_d$ is the standard (unitary) discrete Fourier transform matrix of dimension $d$ defined as 

\begin{equation} \mathcal{F}_d =
    \frac{1}{\sqrt{d}}\begin{bmatrix}
 \omega_d^{0 \cdot 0}     & \omega_d^{0 \cdot 1}     & \cdots & \omega_d^{0 \cdot (d-1)}     \\
 \omega_d^{1 \cdot 0}     & \omega_d^{1 \cdot 1}     & \cdots & \omega_d^{1 \cdot (d-1)}     \\
 \vdots                   & \vdots                   & \ddots & \vdots                       \\
 \omega_d^{(d-1) \cdot 0} & \omega_d^{(d-1) \cdot 1} & \cdots & \omega_d^{(d-1) \cdot (d-1)} \\
\end{bmatrix}, \;\;\;\;\;\;\;\; \omega_d = e^{\frac{-2\pi i}{d}} .
\end{equation}

From this result, it is clear that the desired properties of the Fourier transform and convolution hold.

\subsection{Non-abelian}

\begin{theorem}\label{thm:non-abelian}
Consider a classification task with ground-truth linear predictor $\vbeta$, trained via a linear G-CNN architecture $\on{NN}(\vx) = \langle \vx \star \vw_1 \star \cdots \star  \vw_{L-1} , \vw_L  \rangle$ (see \autoref{sec:main_results} for architecture details) with $L\geq2$ layers under the exponential loss. Then for almost any datasets $\{\vx_i,y_i\}$ separable by $\vbeta$, any bounded sequence of step sizes $\eta_t$, and almost all initializations, suppose that:
\begin{itemize}
    \item The loss $\mathcal{L}(\myW)$ converges to 0
    \item The gradients with respect to the end-to-end linear predictor, $\nabla_{\vbeta} \mathcal{L}(\mathcal{P}(\myW^t) )$, converge in direction as $t \rightarrow \infty$
    \item The iterates $\myW^t$ themselves  converge in direction as $t \rightarrow \infty$ to a separator $\mathcal{P}(\myW^t)$ with positive margin
\end{itemize}
When $L=2$, we need an additional technical assumption, \autoref{assumption:L_2}.
Then, the resultant linear predictor $\mbeta$ is a positive scaling of a first order stationary point of the optimization problem:
\begin{equation} \label{eq:beta_opt}
    \min_\vbeta \| {}\widehat{\vbeta} \|_{2/L}^{(S)} \; \; \text{ s.t. } \; \; \forall n, y_n \left< \vbeta_n, \vx_n \right>_M \geq 1
\end{equation}

\end{theorem}

In this section, we prove the non-abelian case, \autoref{thm:non-abelian}. The proof of our result proceeds according to the following outline:
\begin{enumerate}
    \item By applying a general result of \citet{gunasekar2018linearconv}, \autoref{thm:homog}, we characterize the implicit regularization in the full space of parameters, $\mW$ (in contrast to the end-to-end linear predictor $\beta$), as the stationary point of an optimization problem \autoref{eq:w_opt} in $\mW$.
    \item Separately, we define a \emph{distinct} optimization problem, \autoref{eq:beta_opt} in $\vbeta$. The goal is to demonstrate that stationary points of \autoref{eq:w_opt} are a subset of the stationary points of \autoref{eq:beta_opt}.
    \item The \emph{necessary} KKT conditions for \autoref{eq:w_opt} characterize its stationary points. Using this characterization, we show that the \emph{sufficient} KKT optimality conditions for \autoref{eq:beta_opt} are in fact also satisfied for the corresponding end-to-end predictor. Thus, we show that for any stationary point $\myW^\dagger$ of \autoref{eq:w_opt}, the linear predictor $\cP(\myW^\dagger)$ is a stationary point of \autoref{eq:beta_opt}. 
\end{enumerate}

First, recall that \citet{gunasekar2018linearconv} prove the following general result about the implicit regularization of any homogeneous polynomial parametrization. (\citet{lyu2020gradient} later strengthened this result to the case of arbitrary homogeneous mappings, and showed furthermore that the margin is monotonically increasing under gradient flow, but the result of \citet{gunasekar2018linearconv} is suitable for our needs.)

\begin{theorem}[Homogeneous polynomial parametrization, Theorem 4 of \cite{gunasekar2018linearconv}] \label{thm:homog}
Let $W$ be the concatenation of all (real-valued) parameters $\myW_i$. For any homogeneous polynomial map $\mathcal{P}:\mathbb{R}^P \to \mathbb{R}^{|G|}$ from parameters $\myW_i\in\mathbb{R}^{P}$ to linear predictors, almost all datasets $\{\vx_n, y_n \}_{n=1}^N$ separable by the ground truth predictor $\vbeta := \{ \mathcal{P}(
\myW): \myW \in \mathbb{R}^P \}$, almost all initializations $\myW^0$, and any bounded sequence of step sizes $\{\eta_t\}_t$, consider the gradient descent updates:
\begin{align}
    \myW^{t+1} = \myW^{t} - \eta_t \nabla_{\myW} \mathcal{L}(\myW^t) = \myW^{t} - \eta_t \nabla_{\myW}\mathcal{P}(\myW^t) \nabla_{\beta} \mathcal{L}(\mathcal{P}(\myW^t) )
\end{align}
Suppose furthermore that the exponential loss converges to zero, that the gradients $\nabla_\beta \mathcal{L}(\vbeta^t)$ converge in direction, and that the iterates $W^t$ themselves converge in direction to yield a separator with positive margin. Then, the limit direction of the parameters $\overline \myW^\infty = \lim_{t \to \infty}\frac{\myW^{(t)}}{\|\myW^{(t)}\|_2}$ is a positive scaling of a first order stationary point of the following optimization problem:
\begin{equation} \label{eq:w_opt}
    \min_{\myW \in \mathbb{R}^P} \|\myW\|_2^2 \; \; \; \text{s.t.} \;\;\; \forall n, y_n \langle \vx_n, \mathcal{P}(\myW) \rangle \geq 1.
\end{equation}
To keep track of constant factors, let $\myW^\infty = \tau \overline \myW^\infty$ denote the first order stationary point itself. Furthermore, let $\myW^\infty_i$ denote the individual layers (or parameter blocks) comprising $\myW^\infty$, and similarly let $\overline \myW^\infty_i$ denote the individual layers comprising $\overline \myW^\infty$. We then have via the KKT conditions that: 
\begin{equation} \label{eq:kkt_w}
\begin{split}
    & \exists \{\alpha_n: \alpha_n \geq 0\}_{n=1}^N \; \text{s.t.} \; \alpha_n = 0 \;\; \text{if} \;\; y_n \langle \vx_n, \mathcal{P}(\myW^\infty)\rangle > 1  \\
    & \myW_i^\infty = \nabla_{\myW_i} \mathcal{P}(\myW^\infty) \left[ \sum_n \alpha_n y_n \vx_n \right] = \nabla_{\myW_i} \Big\langle \mathcal{P}(\myW^\infty), \sum_n \alpha_n y_n \vx_n \Big\rangle
\end{split}
\end{equation}
\end{theorem}

While this is an interesting result alone, the goal of implicit regularization is to characterize the final linear predictor (which is some function $\cP$ of the complete parametrization $W$). To that end, consider the following optimization problem in $\vbeta$:

\begin{equation} \label{eq:beta_opt_abelian}
    \min_\vbeta \| {}\widehat{\vbeta} \|_{2/L}^{(S)} \; \; \text{ s.t. } \; \; \forall n, y_n \left< \vbeta_n, x_n \right> \geq 1
\end{equation}

We will leverage the \emph{necessary} KKT conditions from \autoref{eq:kkt_w} to show that first-order stationary points of \autoref{eq:w_opt} are (up to a scaling) also first-order stationary points of \autoref{eq:beta_opt}, using the \emph{sufficient} KKT conditions for \autoref{eq:beta_opt}. 

Using standard KKT sufficiency conditions, the first-order stationary points of \autoref{eq:beta_opt} are those vectors $\vbeta$ such that there exist $\tilde{\alpha}_1,\dots,\tilde{\alpha}_n$ satisfying:
\begin{enumerate}
    \item \textbf{Feasibility: } $\forall n, y_n \left< \vbeta, \vx_n \right> \geq 1$ and $\tilde{\alpha}_i \geq 0 \: \: \forall i$ 
    \item \textbf{Complementary slackness: } $\forall i, \tilde \alpha_i=0$ if $y_n \left< \beta, \vx_n \right> > 1$
    \item \textbf{Membership in subdifferential: }  $\sum_n \tilde \alpha_n y_n {}\widehat{\vx}_n \in \partial^o \| {}\widehat{\vbeta} \|_{2/L}^{(S)}$ 
\end{enumerate}
In the third condition above, $\partial^o$ is the local sub-differential of \cite{clarke1975generalized}: $\partial^o f(\vbeta) = \text{conv}\{ \lim_{i \to \infty} \nabla f(\vbeta+\vh_i): \vh_i \to 0\} $\footnote{$\vh_i$ is a sequence of \emph{vectors} in some linear space, and we take $h_i \to 0$ as an entry-wise statement. This is because the vectors are finite-dimensional, so all norms are equivalent.}. 

We will need the following assumption in the special case $L=2$:
\begin{assumption}[$L=2$ bounded subgradient]\label{assumption:L_2}
Let $\myz = \sum_n \tilde\alpha_n y_n \myx_n$ result from the KKT conditions of the optimization problem in $\myW$, \autoref{eq:kkt_w}, as described previously. Then, we assume that $\| \myz \|_\infty^{(S)} \leq 1 $.
\end{assumption}

Let $\vbeta^\infty = \cP(\myW^\infty)$ and let $\tilde \alpha_i = \frac{1}{\gamma} \alpha_i$ for all $i$, where $\gamma$ is equal to $\Big( || \widehat \vbeta^\infty||_{\frac{2}{L}} \Big)$ for $L>2$ and to $1$ otherwise. 
(Note that by homogeneity of $\cP$, $\cP( \myW^\infty) = \cP(\tau \overline \myW^\infty) =\tau^L \cP(\overline \myW^\infty)$.) We will check these conditions one by one for $\beta^\infty$ and $\tilde \alpha_i$, with the first two following immediately from \autoref{thm:homog} and the last one requiring the most manipulation.
\paragraph{Feasibility} Trivially, $\tilde \alpha_i \geq 0 \: \: \forall i$ by definition of $\alpha$ in \autoref{eq:kkt_w}. Similarly, $y_n\langle \vx_n,\beta^\infty\rangle = y_n\langle \vx_n,\cP(\myW^\infty)\rangle \geq 1$.
\paragraph{Complementary slackness}  If $y_n \left< \beta^\infty, \vx_n \right> > 1 = y_n \left< \cP(\myW^\infty), \vx_n \right> > 1$, then $\tilde \alpha_n \propto \alpha_n = 0$.
\paragraph{Membership in subdifferential} We first characterize the set $\partial^o \| \widehat{\mbet} \|_{2/L}^{(S)}$ for a generic matrix $\mbet$, and then show that $\myz \triangleq \sum_n \tilde \alpha_n y_n \myx_n \in \partial^o \| \mbeta \|_{2/L}^{(S)}$.
When $L=2$ and thus $p=1$, the Schatten norm $\| \mbeta \|_1$ is indeed a norm and its subgradient is known; see e.g. \citet{watson1992characterization}. We restate this result below:

\begin{lemma}[Subdifferential of $p$-Schatten norm, $p=1$]\label{lem:subdifferential_p_1}
    Suppose $L=2$, such that $\frac{2}{L}=p=1$. Let $\mA$ be an $n \times n$ complex-valued matrix.  Then we have
    \begin{equation} 
        \partial^o \| \mA \|_1^{(S)} = \left\{ \mG : \|\mA\|_1^{(S)} = \tr[\mG^\dagger\mA], \|\mG\|_\infty^{(S)} \leq 1]  \right\} 
    \end{equation}
\end{lemma}

When $L>2$, previous works have characterized the subdifferential of the $2/L$-Schatten norm. For example, \citet{lewis1995unitary} characterizes the subdifferential of any unitary matrix norm $\partial ||\mX||$ as all those matrices with the same left and right singular vectors as $\mX$, but whose singular values lie in the subdifferential of the corresponding norm of the singular values of $\mX$ (Corollary 2.5). For our purposes, a different formulation of the subdifferential will be more useful later in the overall proof. To keep the paper relatively self-contained, we prove the following result from scratch. 

\begin{lemma}[Subdifferential of $p$-Schatten norm, $p<1$]\label{lem:subdifferential_p_lessthan_1}
    Suppose $L>2$ and let $p = \frac{2}{L}$, such that $0<p<1$. Let $\mA$ be an $n \times n$ complex-valued matrix with singular value decomposition $\mA=\mU \mD \mV^\dagger$. Let $\Pi_\mV$ project onto the row space of $\mA$, i.e. $\Pi_\mV(\mM) = \mM \mV \mV^\dagger $. Then we have
    \begin{align} \label{eq:subgrad_Schatten}
        \partial^o \| \mA \|_p^{(S)} = &\Bigg\{ \mG : \mA^\dagger \mG  = \frac{1}{\|\mA\|_p^{(S)}}\sqrt{\mA^\dagger \mA}^p \text{ and } \mG \mA^\dagger = \frac{1}{\|\mA\|_p^{(S)}}\sqrt{\mA \mA^\dagger}^p \Bigg\} 
    \end{align}  
\end{lemma}
\begin{proof}

Suppose $\mA$ is rank $r$ and has singular value decomposition $\mA = \sum_{i=1}^r d_i \vu_i \vv_i^\dagger$, where $d_i > 0 $ for all $i$. Consider unit vectors $\myW_1,\dots,\myW_{n-r}$ which are a basis for the orthogonal subspace to $\text{Span}(\vu_1,\dots,\vu_r)$, and unit vectors $\vs_1,\dots,\vs_{n-r}$ which are a basis for the orthogonal subspace to $\text{Span}(\vv_1,\dots,\vv_r)$. Treating the space of $n \times n$ complex matrices as a $n^2$-dimensional linear space, we see that $\{\vu_i\vv_i^\dagger\}_{i=1}^r$ form a basis for an $r$-dimensional subspace. Let $\Pi_A$ denote the projector onto this space, $\Pi_A = \sum_{i=1}^r \vu_i \vv_i^\dagger$. Note that $\Pi_AA = A$.
For any set of $n-r$ sequences $\Big\{\{\epsilon_{i,m}\}_{m=1}^{\infty}\Big\}_{i=1}^{n-r}$ such that $\lim_{m \rightarrow \infty}\epsilon_{m,i} = 0$ for all $i=1,\dots,n-r$, consider the particular sequence of matrices $\{\mH_m\}_{m=1}^\infty$ defined by $$\mH_m = \sum_{i=1}^{n-r} \epsilon_{m,i} \myW_i\vs_i^\dagger$$ By definition, $\lim_{m\rightarrow \infty}||\mH_m||_{\on{Fro}} = 0$, where $
\|\mH_m\|_{\on{Fro}} \triangleq \|\mH_m\|_2^{(S)}$.
 Also, if $\mM$ is a full rank matrix with singular value decomposition $\mA \mD \mB^\dagger$, $||\mM||_p^{(S)}$ is differentiable at $\mM$ and $\nabla ||\mM||_p^{(S)} = \frac{1}{||\mM||_p^{(S)}} \mA \mD^{p-1} \mB^\dagger$. For convenience of notation, let $\mU$ be the matrix with $i^{th}$ column $\vu_i$, $\mW$ the matrix with $i^{th}$ column $\myW_i$, and similarly for $\mV$ and $\mS$ with respect to vectors $\vv_i$ and $\vs_i$ respectively. Also, let $\mD$ be the diagonal matrix with $d_i$ on the $i^{th}$ diagonal.

Combining this fact with the construction of $\mH_m$, we have that 
\begin{align}
    \nabla || \mA + \mH_m||_p^{(S)} &= \nabla \Bigg|\Bigg| \sum_{i=1}^r d_i\vu_i \vv_i^\dagger + \sum_{i=1}^{n-r} \epsilon_{m,i} \myW_i\vs_i^\dagger\Bigg|\Bigg|_p^{(S)} \\
    &= \nabla \Bigg|\Bigg| \begin{bmatrix} \mU & \mW \end{bmatrix} \begin{bmatrix} \mD & 0 & \dots & 0\\
    0 & \epsilon_{m,1} & \dots & 0\\ \vdots & & \ddots & \vdots \\ 0 & 0 & \dots & \epsilon_{m,n-r} \end{bmatrix} \begin{bmatrix} \mV^\dagger \\ \mS^\dagger \end{bmatrix}\Bigg|\Bigg|_p^{(S)} \\
    &= \frac{1}{(\sum_{i=1}^r d_i^p + \sum_{i=1}^{n-r}\epsilon_{m,i}^p)^{\frac{1}{p}}} \Bigg(  \mU \mD^{p-1} \mV^\dagger + \mW \begin{bmatrix} 
     \epsilon_{m,1}^{p-1} & \dots & 0\\ \vdots & & \ddots  \\  0 & \dots & \epsilon_{m,n-r}^{p-1} \end{bmatrix} \mS^\dagger \Bigg)
\end{align}
In the limit as $m$ goes to infinity, $\sum_{i=1}^{n-r}\epsilon_{m,i}^p$ approaches $0$. However, $p-1 < 0$ implies that $\lim_{m \rightarrow \infty} \epsilon_{m,i}^{p-1} = \pm \infty$. By taking convex combinations, one can create any matrix with left and right singular vectors $\mW$ and $\mS^\dagger$. Formally, we have:

\begin{align}
    \partial^o \| \mA\|_p^{(S)} &= \text{conv}\{ \lim_{m \to \infty} \nabla \| \mA+\mH_m\|_p: \mH_m \to 0\}  \\
    &= \frac{1}{||\mA||_p^{(S)}}\mU \mD^{p-1} \mV^\dagger +  \frac{1}{||\mA||_p^{(S)}}\text{conv}\Bigg\{ \lim_{m \to \infty} \mW \begin{bmatrix} 
     \epsilon_{m,1}^{p-1} & \dots & 0\\ \vdots & & \ddots  \\  0 & \dots & \epsilon_{m,n-r}^{p-1} \end{bmatrix} \mS^\dagger: \epsilon_{m,i} \to_m 0\Bigg\} \\
     &= \frac{1}{||\mA||_p^{(S)}}\mU \mD^{p-1} \mV^\dagger +  \{ \mW \boldsymbol{\Sigma} \mS^\dagger: \boldsymbol{\Sigma} \text{ is real and diagonal}\} \label{eq:conv_hull_is_all}
\end{align}
Let $\Pi_\mU$ project onto the column space of $A$, i.e. $\Pi_\mU(\mM) = \mU \mU^\dagger \mM$. Note that for any rank-one matrix $\mM=\va\vb^\dagger$, if $\va$ is not orthogonal to each column of $\mU$, then $\mU\mU^\dagger\mM \neq 0$. Similarly, if $\vb$ is not orthogonal to each column of $\mV$, then $\mM \mV\mV^\dagger \neq 0$. Thus for an arbitrary matrix $\mM$, by decomposing it into a sum of rank-one matrices via its SVD, we see that $\Pi_\mU \mM = \Pi_\mV \mM = 0$ implies that both the row and column spaces of $\mM$ are orthogonal to those of $\mA$, respectively. Thus, we can project via $\Pi_\mU$ and $\Pi_\mV$ to disregard the second term of \autoref{eq:conv_hull_is_all}, and obtain the following expression for the subgradient:
\begin{align}
    \partial^o \| \mA\|_p^{(S)} &= \left\{ \mG : \Pi_\mU \mG =\Pi_\mV \mG=\left(||\mA||_p^{(S)}\right)^{-p}\mU \mD^{p-1} \mV^\dagger \right\}
\end{align}


Consider first the equality
\begin{align}\label{eq:proj_eq_1}
    \Pi_\mU \mG &= \left(||\mA||_p^{(S)}\right)^{-p}\mU \mD^{p-1} \mV^\dagger
\end{align} 
We can left-multiply by $\mA^\dagger$ without changing the set of matrices $\mG$ satisfying this relation. To see this, one can check that if $\mA^\dagger \Pi_\mU \mG =\left(||\mA||_p^{(S)}\right)^{-p} \mA^\dagger\mU \mD^{p-1} \mV^\dagger$, left-multiplying on both sides by $\mU \mD^{-1}\mV^\dagger$ recovers \autoref{eq:proj_eq_1}. 

Similarly, consider the second equality:
\begin{align}\label{eq:proj_eq_2}
    \Pi_\mV \mG =\left(||\mA||_p^{(S)}\right)^{-p}\mU \mD^{p-1} \mV^\dagger 
\end{align} 
We can right-multiply by $\mA^\dagger$ without changing the set of matrices $\mG$ satisfying this relation. To see this, one can check that if $$(\Pi_\mV \mG) \mA^\dagger=\left(||\mA||_p^{(S)}\right)^{-p} \mU \mD^{p-1} \mV^\dagger \mA^\dagger$$ Then right-multiplying on both sides by $\mU \mD^{-1}\mV^\dagger$ recovers \autoref{eq:proj_eq_2}. 
Finally, observe that $\mA^\dagger \Pi_\mU \mG = \mV \mD \mU^\dagger \mU \mU^\dagger \mG = \mA^\dagger \mG$ and $(\Pi_\mV \mG) \mA^\dagger = \mG \mV \mV^\dagger \mV \mD \mU^\dagger = \mG \mV \mD \mU^\dagger = \mG \mA^\dagger$. Furthermore, $\mA^\dagger \frac{1}{||\mA||_p^{(S)}}\mU \mD^{p-1} \mV^\dagger = \mV \mD^p \mV^\dagger = \frac{1}{||\mA||_p^{(S)}}\sqrt{\mA^\dagger \mA}^p$ and $ \frac{1}{||\mA||_p^{(S)}}\mU \mD^{p-1} \mV^\dagger \mA^\dagger = \frac{1}{||\mA||_p^{(S)}}\sqrt{\mA \mA^\dagger}^p$.
This completes the proof of the lemma.
\end{proof}
Lemmas \ref{lem:subdifferential_p_1} and \ref{lem:subdifferential_p_lessthan_1} characterize the subdifferential of the Schatten norm. Now, we show that $\sum_n \tilde \alpha_n y_n \myx_n$ satisfies \autoref{eq:subgrad_Schatten}.



\begin{lemma} \label{lemma:fourier_gradient_ei}
Recall that our G-CNN is given by $\on{NN}(x) = \langle \vx \star \vw_1 \star \cdots \star \vw_{L-1}, \vw_L \rangle$, with the vector of parameters $W = [\vw_1 \: \dots \vw_L]$ and end-to-end linear predictor given by $\mathcal{P}(\myW) = \vw_1 \star \dots \star \vw_L$. Consider an arbitrary such vector of real-valued parameters. Also, we have assumed that the filters in real-space are real-valued, i.e. $\vw_i \in \mathbb{R}^{|G|}$. Then the following relation holds:
\begin{equation}\cF_M \nabla_{\vw_\ell} \langle \mathcal{P}(\myW), \ve_i \rangle = {}\widehat{\vw}_{\ell+1}^\dagger \cdots \widehat{\vw}_{L}^\dagger {}\widehat{\ve}_i {}\widehat{\vw}_1^\dagger \cdots \widehat{\vw}_{\ell-1}^\dagger
\end{equation}
where $e_i$ is the $i^{th}$ standard basis vector. 
\end{lemma}
\begin{proof}
Since $ \langle \mathcal{P}(\myW), \ve_i \rangle$ is real, we have $\langle \mathcal{P}(\myW), \ve_i \rangle = \langle \ve_i, \mathcal{P}(\myW)\rangle$. Plugging this in,
\begin{equation}
    \begin{split}
       \cF_M \nabla_{\vw_\ell} \mathcal{P}(\myW) [\ve_i] &= \cF_M \nabla_{\vw_\ell} \langle \mathcal{P}(\myW), \ve_i \rangle \\
        &=\cF_M \nabla_{\vw_\ell} \langle \ve_i, \mathcal{P}(\myW) \rangle \\
        &=\cF_M \nabla_{\vw_\ell} \frac{1}{|G|} \langle\cF_M \ve_i,\cF_M \mathcal{P}(\myW) \rangle_M \\
        &= \cF_M \nabla_{\vw_\ell} \frac{1}{|G|} \langle  {}\widehat{\ve}_i, {}\widehat{\vw}_L \cdots \widehat{\vw}_1 \rangle_M \\
        &=\cF_M \nabla_{\vw_\ell} \frac{1}{|G|} \tr[  {}\widehat{\ve}_i ({}\widehat{\vw}_L \cdots \widehat{\vw}_1)^\dagger ] \\
        &=\cF_M \nabla_{\vw_\ell} \frac{1}{|G|} \tr[ {}\widehat{\vw}_{\ell+1}^\dagger \cdots \widehat{\vw}_{L}^\dagger {}\widehat{\ve}_i {}\widehat{\vw}_1^\dagger \cdots \widehat{\vw}_\ell^\dagger  ] \\
        &=\cF_M \nabla_{\vw_\ell} \frac{1}{|G|} \langle {}\widehat{\vw}_{\ell+1}^\dagger \cdots \widehat{\vw}_{L}^\dagger {}\widehat{\ve}_i {}\widehat{\vw}_1^\dagger \cdots \widehat{\vw}_{\ell-1}^\dagger, {}\widehat{\vw}_\ell \rangle_M \\
        &=\cF_M \nabla_{\vw_\ell} \left\langle\cF_M^{-1} ({}\widehat{\vw}_{\ell+1}^\dagger \cdots \widehat{\vw}_{L}^\dagger {}\widehat{\ve}_i {}\widehat{w}_1^\dagger \cdots \widehat{\vw}_{\ell-1}^\dagger), \vw_\ell \right\rangle \\
        &= {}\widehat{\vw}_{\ell+1}^\dagger \cdots \widehat{\vw}_{L}^\dagger {}\widehat{\ve}_i {}\widehat{\vw}_1^\dagger \cdots \widehat{\vw}_{\ell-1}^\dagger
    \end{split}
\end{equation}
\end{proof}
Letting $\vz = \sum_n \tilde \alpha_n y_n  \vx_n$ and $\myW=\myW^\infty$, \autoref{lemma:fourier_gradient_ei} implies that 

\begin{align}
   \cF_M \nabla_{ \vw_\ell^\infty} \langle \mathcal{P}( \myW^\infty), \vz \rangle = {{}\widehat{\vw}_{\ell+1}^\infty}^\dagger \cdots {{}\widehat{\vw}_{L}^\infty}^\dagger {}\widehat{\vz} {{}\widehat{\vw}_1^\infty}^\dagger \cdots {{}\widehat{\vw}_{\ell-1}^\infty}^\dagger
\end{align}

By combining \autoref{eq:kkt_w} with \autoref{lemma:fourier_gradient_ei}, we have 
\begin{align}
    \frac{1}{\gamma} {}\widehat \vw_\ell^\infty &= \frac{1}{\gamma} \nabla_{\vw_\ell^\infty} \Big\langle \mathcal{P}(\myW^\infty), \sum_n \alpha_n y_n \vx_n \Big\rangle\\
    &=  \nabla_{\vw_\ell^\infty} \Big\langle \mathcal{P}(\myW^\infty), \sum_n \tilde \alpha_n y_n \vx_n \Big\rangle \\
    &= {}\widehat{\vw}_{\ell+1}^{\infty\dagger} \cdots \widehat{\vw}_{L}^{\infty^\dagger} {}\widehat{\vz} {}\widehat{\vw}_1^{\infty\dagger} \cdots \widehat{\vw}_{\ell-1}^{\infty\dagger}
\end{align} 

As a result:
\begin{equation}\label{eq:wl_wl}
    \begin{split}
        {}\widehat{\vw}_\ell^\infty &= \gamma {}\widehat{\vw}_{\ell+1}^{\infty\dagger} \cdots \widehat{\vw}_L^{\infty\dagger}  {}\widehat{\vz} {}\widehat{\vw}_{1}^{\infty\dagger} \cdots \widehat{\vw}_{\ell-1}^{\infty\dagger} \\
        {}\widehat{\vw}_\ell^\infty {}\widehat{\vw}_\ell^{\infty\dagger} &= \gamma {}\widehat{\vw}_{\ell+1}^{\infty\dagger} \cdots \widehat{\vw}_L^{\infty\dagger}  {}\widehat{\vz} {}\widehat{\vw}_{1}^{\infty\dagger} \cdots \widehat{\vw}_{\ell}^{\infty\dagger}
    \end{split}
\end{equation}

Applying this relation with $\ell=L$, we have that
\begin{align}\label{eq:hermitian}
    {}\widehat{\vw}_L^\infty {}\widehat{\vw}_L^{\infty\dagger} &= \gamma {}\widehat{\vz} {}\widehat{\vw}_{1}^{\infty\dagger} \cdots \widehat{\vw}_{L}^{\infty\dagger} \\
    &= \gamma {}\widehat{\vz} {}\widehat \vbeta^{\infty \dagger}
\end{align}
Taking adjoints of both sides implies that ${}\widehat{\vz} {}\widehat \vbeta^{\infty \dagger}$ is Hermitian, which will be useful later. 

Let $\mbeta =\cF_M\mathcal{P}(\myW^\infty)$, from which we can derive the following recursion:
\begin{equation}
    \begin{split}
        {}\widehat{\vbeta}^\infty \mbetdag &=    {}\widehat{\vw}_L^\infty  \cdots \widehat{\vw}_1^\infty   {}\widehat{\vw}_{1}^{\infty\dagger} \cdots \widehat{\vw}_{L}^{\infty\dagger}   \\
        &= \gamma^1  {}\widehat{\vw}_L^\infty  \cdots  {}\widehat{\vw}_{2}^{\infty}  {}\widehat{\vw}_{2}^{\infty\dagger} \cdots  {}\widehat{\vw}_{L}^{\infty\dagger} {}\widehat{\vz}  {}\widehat{\vw}_{1}^{\infty\dagger} \cdots \widehat{\vw}_{L}^{\infty\dagger}   \: \text{ by  \autoref{eq:wl_wl}} \\
        &= \gamma^2   {}\widehat{\vw}_L^\infty  \cdots  {}\widehat{\vw}_{3}^{\infty} {}\widehat{\vw}_{3}^{\infty\dagger} \cdots  {}\widehat{\vw}_{L}^{\infty\dagger} {}\widehat{\vz}  {}\widehat{\vw}_{1}^{\infty\dagger}  {}\widehat{\vw}_{2}^{\infty\dagger} \cdots  {}\widehat{\vw}_{L}^{\infty\dagger} {}\widehat{\vz}  {}\widehat{\vw}_{1}^{\infty\dagger} \cdots \widehat{\vw}_{L}^{\infty\dagger}   \: \text{ again, by  \autoref{eq:wl_wl}} \\
        &= \gamma^2   {}\widehat{\vw}_L^\infty  \cdots  {}\widehat{\vw}_{3}^{\infty} {}\widehat{\vw}_{3}^{\infty\dagger} \cdots  {}\widehat{\vw}_{L}^{\infty\dagger} {}\widehat{\vz} \mbetdag {}\widehat{\vz} \mbetdag    \: \text{ by definition of $\vbeta^\infty$} \\
        &= \gamma^2  {}\widehat{\vw}_L^\infty  \cdots {}\widehat{\vw}_{3}^{\infty} {}\widehat{\vw}_{3}^{\infty\dagger} \cdots  {}\widehat{\vw}_{L}^{\infty\dagger} ({}\widehat{\vz} \mbetdag)^2  \text{ by repeated application of \autoref{eq:wl_wl} }\\
        &= \gamma^{L}  ({}\widehat{\vz} \mbetdag)^L \label{eq:beta_recursion} \\
    \end{split}
\end{equation}

Similarly to \autoref{eq:wl_wl}, we have:
\begin{align}\label{eq:wl_reversed}
    \mwl{\ell}^\dagger \mwl{\ell} &= \gamma \mwld{\ell}\mwld{\ell+1}\dots\mwld{L}\myz \mwld{1}\dots\mwld{\ell-1}
\end{align}

By considering $\ell=1$, we have $\mwl{1}^\dagger \mwl{1} = \mbetdag \myz$, which shows that $\mbetdag \myz$ is Hermitian as well.

Using \autoref{eq:wl_reversed}, we can similarly reason about $\mbetdag \mbeta$:
\begin{equation}\begin{split}\label{eq:beta_recursion_2}
    \mbetdag \mbeta &= \mwld{1}\dots\mwld{L} \mwl{L}\dots\mwl{1} \\
    &= \gamma \mwld{1}\dots\mwld{L-1} (\mwld{L}\myz\mwld{1}\dots\mwld{L-1}) \mwl{L-1}\dots\mwl{1} \\
    &= \gamma (\mbetdag \myz) \mwld{1}\dots\mwld{L-1} \mwl{L-1}\dots\mwl{1} \\
    &= \gamma (\mbetdag \myz) \mwld{1}\dots\mwld{L-1}\mwld{L}\myz \mwld{1}\dots\mwld{L-2} \mwl{L-2}\dots\mwl{1} \\
    &= \gamma^2 (\mbetdag \myz)^2 \mwld{1}\dots\mwld{L-2} \mwl{L-2}\dots\mwl{1} \\
    &= \gamma^L (\mbetdag \myz)^L 
\end{split}\end{equation}

For $L>2$, we have shown in \autoref{lem:subdifferential_p_lessthan_1} that 
\begin{align} \label{eq:subgrad_at_beta}
        \partial^o \| \mbeta \|_p^{(S)} = &\Bigg\{ \mG : \mbetdag \mG  = \frac{1}{||\mbeta||_p^{(S)}}\sqrt{\mbetdag \mbeta}^p \text{ and } \mG \mbetdag = \frac{1}{||\mbeta||_p^{(S)}}\sqrt{\mbeta \mbetdag}^p \Bigg\} 
\end{align}
Let's check that setting $G = \widehat \vz$ satisfies this relation. Since $p=\frac{2}{L}$, $\frac{p}{2}=\frac{1}{L}$ and, by \autoref{eq:beta_recursion} and that $\myz \mbetdag$ is Hermitian: 
\begin{align}
    (\mbeta \mbetdag)^{p/2} = \gamma \myz \mbetdag
\end{align}
Similarly, by \autoref{eq:beta_recursion_2} and that $\mbetdag \myz$ is Hermitian:
\begin{align}
    (\mbeta \mbetdag)^{\frac{p}{2}} &= \gamma \mbetdag \myz
\end{align}
By choice of $\gamma$, $\gamma = \| \mbeta \|_{\frac{2}{L}}^{(S)}$. Thus, $\widehat \vz \in \partial^o \| \mbeta \|_p^{(S)}$ as desired for $L>2$.

For $L=2$, $p=1$ and by \autoref{lem:subdifferential_p_1} we had the following expression for the subgradient:

\begin{equation} 
        \partial^o \| \mA \|_1^{(S)} = \left\{ \mG : \|\mA\|_1^{(S)} = \tr[\mG^\dagger\mA], \|\mG\|_\infty^{(S)} \leq 1]  \right\} 
\end{equation}


As a technical condition, we had to assume in \autoref{assumption:L_2} that $\| \myz \|_\infty^{(S)} \leq 1 $. (We believe that with a refined analysis in future work, this assumption can be shown to be true given only our existing assumptions.) By the previous reasoning for $L=2$, we had that $\mbetdag \mbeta = \gamma^2 (\mbetdag \myz)^2$. Since $\mbetdag \mbeta = \mU \mD^2 \mU^\dagger$ is positive semi-definite and symmetric, and since $\mbetdag \myz$ is Hermitian, we can take the square root of both sides and obtain that $\mU \mD \mU^\dagger = \gamma \mbetdag \myz = \gamma  \myz^\dagger \mbeta$. Thus, $\tr[ \myz^\dagger \mbeta] = \tr[\mD \mU^\dagger \mU] = \tr[\mD] = \| \mbeta \| _1^{(S)}$, which is what was needed (since $\gamma=1$ for the $L=2$ case).

\subsubsection{Infinite groups with band-limited inputs}\label{subsubsec:infinite}

Here we consider the case of more general, infinite-dimensional compact Lie groups. Such groups admit a Fourier transform which is an operator between infinite-dimensional spaces, rather than a finite matrix as before, but which has the same key properties: a convolution theorem, and preservation of inner products. To be concrete, the Fourier transform is now defined as $\widehat{f}(\rho) = \int_{g \in G} \rho(g) f(g) \mu(g) $, where $\mu(g)$ denotes the Haar measure for the group. 


Since it is impossible to store a general function $x:G\rightarrow \R$, one must make simplifying assumptions on the input $x$. A common assumption is that $x$ is band-limited in Fourier space, i.e. $\widehat{x}$ is supported only on a small (and finite) subset of Fourier coefficients contained within the irreps $\rho \in S$. For many such groups, there is a natural hierarchy of irreps (analogous to low frequencies and high frequencies in classical Fourier analysis), and so practical architectures typically assume only those corresponding to low frequencies are non-zero.

For ease of analysis, we will assume that the input functions and all convolutional filters are real-valued in Fourier space. The architecture of our G-CNN is the same as in the finite-dimensional group setting except that we will apply functions entirely in the finite-dimensional Fourier space over the irreps in $S$. Given a function $\underline{\widehat{\vx}}$ supported only on $S$ in Fourier space, we run gradient descent over only the Fourier coefficients on $S$ of filters $\widehat{\vw_i}$, and assume they are zero elsewhere. Before we proceed, we define $\mathcal{F}_M$ in the natural way after restricting the irreps to the band-limiting space $S$:
\begin{align}
\underline{\widehat \vf} = \cF_M \vf  = \bigoplus_{\rho \in {}S} \widehat{f}(\rho)^{\oplus d_\rho} \;\in\on{GL}\left(\sum_{\rho \in S}d_\rho^2, \,\mathbb{C}\right).
\end{align}
In this band-limited space, for each filter, there are $\sum_{\rho \in S}d_\rho^2$ trainable parameters corresponding to the entries of the irreps in $S$. Note that these entries are also orthogonal to each other with respect to the inner product $\langle \cdot,\cdot \rangle_M$ and thus form a vector space of dimension $\sum_{\rho \in S}d_\rho^2$.

Recall that we previously applied \autoref{thm:homog} to the homogeneous polynomial parametrization from $(\vw_1,\dots,\vw_L) \rightarrow \langle \vx \ast \vw_1 \ast \dots \ast \vw_{L-1}, \vw_L\rangle$. Since we will operate in Fourier space, instead consider the homogeneous polynomial parametrization $\widehat{\myW} \in \R^p$ containing the parameters stored in the matrices $\{ \underline{\widehat{\vw_1}},\dots,\underline{\widehat{\vw_L}} \}$. In other words, there are $p = L(\sum_{\rho \in S}d_\rho^2)$ parameters stored in the matrices contained in the set $\widehat{\myW}$. 
\begin{align}
    \widehat{\myW}^{t+1} = \widehat{\myW}^{t} - \eta_t \nabla_{\widehat{\myW}} \mathcal{L}(\widehat{\myW}^t) = \widehat{\myW}^{t} - \eta_t \nabla_{\widehat{\myW}}\mathcal{P}(\widehat{\myW}^t) \nabla_{\widehat{\vbeta}} \mathcal{L}(\mathcal{P}(\widehat{\myW}^t) )
\end{align}
Note that here, iterates are only allowed to vary over the finite subset $S$ of Fourier coefficients and are assumed to be equal to $0$ elsewhere. If we assume further that the exponential loss converges to zero, that the gradients $\nabla_{\widehat{\vbeta}} \mathcal{L}(\widehat{\vbeta}^t)$ converge in direction, and that the iterates $\widehat{\myW^t}$ themselves converge in direction to yield a separator with positive margin, then the limit direction of the parameters $ \overline{\widehat{\myW^{\infty}}}= \lim_{t \to \infty}\frac{\widehat{\myW^{t}}}{\|\widehat{\myW^{t}}\|_2} $ is a positive scaling of a first order stationary point of the following optimization problem:
\begin{equation} \label{eq:infinite_w_opt}
    \min_{\widehat{\myW} \in \mathbb{R}^P} \|\widehat{\myW}\|_2^2 \; \; \; \text{s.t.} \;\;\; \forall n, y_n \langle \widehat{\vx_n}, \mathcal{P}(\widehat{\myW}) \rangle_M \geq 1.
\end{equation}
Again letting $\widehat{\myW^\infty} = \tau \overline{\widehat{\myW^{\infty}}}$ denote the first order stationary point itself, $\widehat{\myW^\infty}_i$ the individual layers (or parameter blocks) comprising $\widehat{\myW^\infty}$, and $\overline{\widehat{\myW^{\infty}}}_i$ the individual layers comprising $\overline{\widehat{\myW^{\infty}}}$, we again have via the KKT conditions that: 
\begin{equation}
\begin{split}
    & \exists \{\alpha_n: \alpha_n \geq 0\}_{n=1}^N \; \text{s.t.} \; \alpha_n = 0 \;\; \text{if} \;\; y_n \langle \underline{\widehat{\vx_n}}, \mathcal{P}(\widehat{\myW^{\infty}})\rangle_M > 1  \\
    & \widehat{\myW_i}^\infty = \nabla_{\widehat{\myW_i}} \mathcal{P}(\widehat{\myW^{\infty}}) \left[ \sum_n \alpha_n y_n \underline{\widehat{\vx_n}} \right] = \nabla_{\widehat{\myW_i}} \Big\langle \mathcal{P}(\widehat{\myW^{\infty}}), \sum_n \alpha_n y_n \underline{\widehat{\vx_n}} \Big\rangle_M
\end{split}
\end{equation}

Equivalently, writing the above in terms of the matrices $\underline{\widehat{\vw}}_\ell^\infty$, and defining $\underline{\widehat{\vz}} = \sum_n \alpha_n y_n \underline{\widehat{\vx_n}}$, we have an equivalence to \autoref{lemma:fourier_gradient_ei}.

\begin{equation}
    \begin{split}
       \nabla_{\vw_\ell} \mathcal{P}(\widehat{\myW^{\infty}}) [\underline{\widehat{\vz}}] &= \nabla_{\underline{\widehat{\vw}}^\infty_\ell} \langle \mathcal{P}(\widehat{\myW^{\infty}}), \underline{\widehat{\vz}} \rangle_M \\
        &=\nabla_{\underline{\widehat{\vw}}^\infty_\ell} \langle \underline{\widehat{\vz}}, \mathcal{P}(\widehat{\myW}^\infty) \rangle_M \\
        &= \nabla_{\underline{\widehat{\vw}}^\infty_\ell} \langle  {}\underline{\widehat{\vz}}, {}\underline{\widehat{\vw}}^\infty_L \cdots \underline{\widehat{\vw}}^\infty_1 \rangle_M \\
        &=\nabla_{\underline{\widehat{\vw}}^\infty_\ell} \tr[  {}\underline{\widehat{\vz}} ({}\underline{\widehat{\vw}}^\infty_L \cdots \underline{\widehat{\vw}}^{\infty ^\dagger}_1) ] \\
        &=\nabla_{\underline{\widehat{\vw}}^\infty_\ell} \tr[ {}\underline{\widehat{\vw}}^{\infty ^\dagger}_{\ell+1} \cdots \underline{\widehat{\vw}}^{\infty ^\dagger}_{L} {}\underline{\widehat{\vz}} {}\underline{\widehat{\vw}}^{\infty ^\dagger}_1 \cdots \underline{\widehat{\vw}}^{\infty ^\dagger}_\ell  ] \\
        &=\nabla_{\underline{\widehat{\vw}}^\infty_\ell} \langle {}\underline{\widehat{\vw}}^{\infty ^\dagger}_{\ell+1} \cdots \underline{\widehat{\vw}}^{\infty ^\dagger}_{L} {}\underline{\widehat{\vz}} {}\underline{\widehat{\vw}}^{\infty ^\dagger}_1 \cdots \underline{\widehat{\vw}}^{\infty ^\dagger}_{\ell-1}, {}\underline{\widehat{\vw}}^\infty_\ell \rangle_M \\
        &= {}\underline{\widehat{\vw}}^{\infty ^\dagger}_{\ell+1} \cdots \underline{\widehat{\vw}}^{\infty \dagger}_{L} {}\underline{\widehat{\vz}} {}\underline{\widehat{\vw}}^{\infty ^\dagger}_1 \cdots \underline{\widehat{\vw}}^{\infty ^\dagger}_{\ell-1}
    \end{split}
\end{equation}

Using the KKT conditions above and this fact, all of the manipulations demonstrating that a positive scaling of $\sum_n \alpha_n y_n \underline{\widehat{\vx_n}}$ is a first-order stationary point of the optimization problem below carry over exactly as they do in the previous part. This yields the following formal result:

\begin{theorem}\label{thm:infinite-formal}
Consider a classification task with ground-truth linear predictor $\underline{\widehat{\vbeta}}$, trained via a real-valued, Fourier-space, band-limited linear G-CNN architecture $\on{NN}(\underline{\widehat{\vx}}) = \langle \underline{\widehat{\vx}} , \underline{\widehat{\vw_L}}\dots\underline{\widehat{\vw_1}} \rangle$ with $L\geq2$ layers under the exponential loss. Then for almost any datasets $\{\vx_i,y_i\}$ separable by $\vbeta$, any bounded sequence of step sizes $\eta_t$, and almost all initializations, suppose that:
\begin{itemize}
    \item The loss converges to 0
    \item The gradients with respect to the end-to-end linear predictor $\underline{\widehat{\vbeta}}$ converge in direction
    \item The iterates $\underline{\widehat{\vw_i}}$ themselves converge in direction to a separator with positive margin
\end{itemize}
When $L=2$, we need an additional technical assumption, \autoref{assumption:L_2}.
Then, the resultant linear predictor $\underline{\mbeta}$ is a positive scaling of a first order stationary point of the optimization problem:
\begin{equation} 
    \min_{\underline{\widehat{\vbeta}}} \Big\| \underline{\widehat{\vbeta}} \Big\|^{(S)}_{2/L} \; \; \text{ s.t. } \; \; \forall n, y_n \left< \underline{\widehat{\vx_n}}, \underline{\widehat {\vbeta}} \right>_M \geq 1.
\end{equation}
\end{theorem}

\section{Group Fourier Transforms}
\label{app:gfourier}

To aid the reader in understanding the notation and structure behind the group Fourier transform, the following exposition is given for reference and convenience. Here, we introduce important concepts from representation theory and from there, provide explicit constructions the group Fourier transform.

A representation of a group $G$ is a vector space $V$ together with a $G$-linear map $\rho: G \longrightarrow \on{GL}(V)$. Of particular interest is the \emph{regular representation} which we construct as follows. Let $G$ be a group of order $n$ and choose $V = \mathbb{C}^n$. Consider an element $u \in \mathbb{C}[G]$, so $u = a_1 g_1 + \cdots + a_n g_n$ where $a_i \in \mathbb{C}$ and $g_n \in G$, and the associated vector $\mathbf{u} \in \mathbb{C}^n$ such that $\mathbf{u} = \begin{bmatrix}a_1 & \cdots & a_n\end{bmatrix}$.

The action of left-multiplication of $u$ for any $h \in G$ yields $hu = a_1(hg_1) + \cdots + a_n(hg_n)$, which is equivalent to a permutation of the coefficients, so there is a unique matrix $H \in \on{GL}(\mathbb{C}^n)$ such that $H\mathbf{u}$ is equivalent to the associated vector for $hu$. The $G$-linear map $L_h: h \mapsto H$ is the (left) \emph{regular representation}.

The direct sum of two $G$-representations $(\rho_1, V_1)$ and $(\rho_2, V_2)$ can be constructed by
\begin{align}
    (\rho_1 \oplus \rho_2) (g) = \begin{bmatrix}\rho_1(g)&\\&\rho_2(g)\end{bmatrix} 
    \label{eq:reducible_representation}
\end{align}

The dimension $d_\rho$ of a representation $(\rho, V)$ is defined to be $\dim(V)$. A finite-dimensional representation is \emph{irreducible} if it cannot be written as the direct sum of two nontrivial representations. Denote $\widehat G$ to be the set of isomorphism classes of irreducible subrepresentations of $L_u$, and let $d_{\rho_i}$ denote the dimension of $\rho_i \in \widehat G$. As it turns out, there is an isomorphism $$L_u \cong \bigoplus_{\rho \in \widehat{G}} \rho^{\oplus d_\rho}$$
Where $\rho$ ranges over one representative from each of the isomorphism classes of $\widehat G$, repeated according to its multiplicity. In general, this decomposition is not uniquely determined as it depends on the choice of representatives. Throughout this paper, we choose representatives such that each $\rho$ is \emph{unitary}, meaning that every $\rho(g)$ is a unitary matrix.

Every function $f: G \longrightarrow \mathbb{C}$ can be considered as equivalent to an element $u_f \in \mathbb{C}[G]$ by setting $u_f = f(g_1)g_1 + \cdots + f(g_n) g_n$. And as we have already seen, every $u \in \mathbb{C}[G]$ can naturally be considered a subrepresentation of $L_u$. Then the \emph{Fourier transform} of $f$ at a representation $\rho$, denoted $\widehat{f}(\rho) \in \on{GL}(d_\rho, \mathbb{C})$ where $\widehat{f}(\rho) = \sum_{u \in G} f(u) \rho(u)$, can be considered as a projection of $L_u(u_f)$ onto the orthogonal subspace described by $\rho$. Throughout the paper we use slightly different notations and characterizations of the Fourier transform depending on the context, but all share this projection as the fundamental operation.\\

Recalling \autoref{eq:reducible_representation}, there is a representation isomorphic to $L_u$, which we will suggestively call $\cF_M$, that block-diagonalizes the image of $L_u$ into orthogonal subspaces of unitary irreducible representations (each $\rho \in \widehat{G}$ repeated with multiplicity $d_\rho$):
\begin{align}
    \cF_M(u) = \begin{bmatrix}\rho_1(u)&&\\&\ddots&&\\&&\rho_j(g)\end{bmatrix} \in \on{GL}(\mathbb{C}^n)
\end{align}
For the last piece of the puzzle, extend the domain of $\cF_M$ to all functions $f:G \longrightarrow \mathbb{C}$ (by considering $f$ as an element of $\mathbb{C}[G]$). Then we obtain
\begin{align}
    \cF_Mf = \begin{bmatrix}\widehat{f}(\rho_1)&&\\&\ddots&\\&&\widehat{f}(\rho_j)\end{bmatrix} \in \on{GL}(\mathbb{C}^n)
\end{align}
Which we call the \emph{matrix Fourier transform} of $f$.

We also make use of the \emph{Fourier basis matrix} $\mathcal{F}$, which depends only on the group $G$ and not the function $f$. To construct it, we first need the operation $\on{Flatten}$ which vertically stacks the columns of a matrix. Then define the transform
\begin{align}
    \phi(f) = \begin{bmatrix}\on{Flatten}\left(\widehat{f}(\rho_1)\right)\\
                    \vdots\\
                    \on{Flatten}\left(\widehat{f}(\rho_j)\right)
    \end{bmatrix}
\end{align}
Letting $e_i: G \longrightarrow \mathbb{C}$ be the indicator function $e_i(g_j) = \mathbbm{1}_{i = j}$ we can describe the unitary Fourier basis matrix $\mathcal{F}$ for a group $G$ as a row-scaling of
\begin{align}
    \mathcal{F} \propto \begin{bmatrix}\phi(e_1) & \phi(e_2) & \cdots & \phi(e_n)\end{bmatrix}
\end{align}

In other words, given a column vectorization $\mathbf{f}$ of a function $f$ such that $\mathbf{f}_i = f(g_i)$, then $\mathcal{F}$ is the matrix such that the `unflattening' of $\mathcal{F}\mathbf{f}$ is equal to $\cF_Mf$ up to the row-scaling constants. Thus we can treat the group Fourier transform either as an abstract isomorphism or as a concrete matrix-vector multiplication, depending on the application.

The matrix $\mathcal{F}$ can be explicitly constructed as described in \autoref{def:group_fourier_transform}. Denoting $\ve_{[\rho,i,j]}$ as the column-major vectorized basis for element $\rho_{ij}$ in the group Fourier transform, then we can form the matrix
\begin{equation}
    \cF = \sum_{u \in G} \sum_{\rho \in {}\widehat{G}} \frac{\sqrt{d_\rho}}{\sqrt{|G|}} \sum_{i,j=1}^{d_\rho} \rho(u)_{ij} \ve_{[\rho,i,j]} \ve_u^{T}.
\end{equation}

For visualization, consider the dihedral group $D_6 = \{1,\,r,\,r^2,\,a,\,ar,\,ar^2\}$, which has three irreducible representations $\rho_1, \rho_2,\rho_3$ (up to isomorphism) of dimensions $1$, $1$, and $2$ respectively, and let $f: D_6 \longrightarrow \mathbb{C}$. Using colors instead of values at first to avoid numerical clutter:

\begin{align*}
    \widehat{f}(\rho_1) = \left[\begin{array}{c}
    \aaa
  \end{array}\right]&& \widehat{f}(\rho_2) = \left[\begin{array}{c}
    \bbb
  \end{array}\right]&&\widehat{f}(\rho_3) = \left[\begin{array}{c|c}
    \ccc & \ccc \\
    \hline
    \ccc & \ccc
  \end{array}\right]
\end{align*}

Since $L_{D_6} \cong \rho_1\oplus\rho_2\oplus\rho_3^2$, we can get something like
\begin{align*}
    \cF_Mf = \left[\begin{array}[text width = 10mm]{cccccc}
    \aaa\,\,&&&&&\\
    &\bbb\,\,&&&&\\
    &&\ccc&\ccc\,\,\,&&\\
    &&\ccc&\ccc\,\,\,&&\\
    &&&&\ccc&\ccc\,\,\\
    &&&&\ccc\,\,&\ccc
  \end{array}\right] \in \on{GL}(\mathbb{C}^n)
\end{align*}
Whereas for the unitary Fourier basis matrix we have the form

\begin{align*}
&\begin{tabular}{p{1mm}p{0.2mm}p{0.2mm}p{0.1mm}p{0.1mm}p{0.1mm}}
     $\scriptstyle{1}$&$\scriptstyle{r}$&$\scriptstyle{r^2}$&$\scriptstyle{a}$&$\scriptstyle{ar}$&$\scriptstyle{ar^2}$
\end{tabular}\\
\mathcal{F} \propto \begin{array}{c}
\scriptstyle{\rho_1}\\
\scriptstyle{\rho_2}\\
\scriptstyle{[\rho_3]_{11}}\\
\scriptstyle{[\rho_3]_{21}}\\
\scriptstyle{[\rho_3]_{12}}\\
\scriptstyle{[\rho_3}]_{22}
\end{array}\hspace{-2.4mm}&\left[\begin{array}{cccccc}
    \rowcolor{cyan!90}&\,\,&\,\,\,&\,\,&\,\,&\,\,\\
    \rowcolor{orange!90}&\,\,&\,\,&\,\,&\,\,&\,\,\\
    \rowcolor{green!90}&\,\,&\,\,&\,\,&\,\,&\,\,\\
    \rowcolor{green!90}&\,\,&\,\,&\,\,&\,\,&\,\,\\
    \hline
    \rowcolor{green!90}&\,\,&\,\,&\,\,&\,\,&\,\,\\
    \rowcolor{green!90}&\,\,&\,\,&\,\,&\,\,&\,\,\\
  \end{array}\right]
\end{align*}

Note that we do not yet include the row-scaling constants. Now explicitly, choosing $\rho_1$ the trivial irrep, $\rho_2$ the sign irrep, and $\rho_3$ the representation sending
\begin{align}
    \rho_3(a) = \begin{bmatrix}0&1\\1&0\end{bmatrix} &&\rho_3(r) = \begin{bmatrix}\omega&0\\0&\omega^2\end{bmatrix}
\end{align}
Where $\omega = e^{2\pi i/3}$, then the matrix $\cF$ is exactly

\begin{align}\label{eq:dihedral_irreps}
&\begin{tabular}{p{10mm}p{10mm}p{6mm}p{6mm}p{6mm}p{6mm}}
     $\scriptstyle{1}$&$\scriptstyle{r}$&$\scriptstyle{r^2}$&$\scriptstyle{a}$&$\scriptstyle{ar}$&$\scriptstyle{ar^2}$
\end{tabular}\nonumber\\
\mathcal{F} = \frac{1}{\sqrt{6}} \begin{array}{c}
\scriptstyle{\rho_1}\\
\scriptstyle{\rho_2}\\
\scriptstyle{[\rho_3]_{11}}\\
\scriptstyle{[\rho_3]_{21}}\\
\scriptstyle{[\rho_3]_{12}}\\
\scriptstyle{[\rho_3}]_{22}
\end{array}\hspace{-2.4mm}&\left[\begin{array}{cccccc}
    \rowcolor{cyan!90}1&1&1&1&1&1\\
    \rowcolor{orange!90}1&1&1&-1&-1&-1\\
    \rowcolor{green!90}\sqrt{2}&\sqrt{2}\omega&\sqrt{2}\omega^2&0&0&0\\
    \rowcolor{green!90}0&0&0&\sqrt{2}&\sqrt{2}\omega&\sqrt{2}\omega^2\\
    \hline
    \rowcolor{green!90}0&0&0&\sqrt{2}\omega^2&\sqrt{2}\omega&\sqrt{2}\\
    \rowcolor{green!90}\sqrt{2}&\sqrt{2}\omega^{2}&\sqrt{2}\omega&0&0&0\\
  \end{array}\right]
\end{align}

Note that the above is only one possible way of writing $\cF$ since the $2$-dimensional irreducible representation of $D_6$ is unique only up to conjugation by a unitary matrix.

\section{Uncertainty principles for groups}
\label{app:uncertainty}
In mathematics, uncertainty principles categorize trade-offs of the ``amount of information" stored in a function between canonically conjugate regimes, \textit{e.g.,} position (real regime) and momentum (Fourier regime) of a physical particle. More generically, uncertainty principles show that a function and its Fourier transform cannot both be very localized or concentrated. In the context of group theory, uncertainty principles specifically show that sparse support in either the real or Fourier regime of a group necessarily implies dense support in the conjugate regime \citep{wigderson2021uncertainty}. Such results are directly relevant when interpreting implicit regularization of linear G-CNNs which bias gradient descent towards sparse solution in the Fourier basis of the group. In this section, we formally state and summarize these group theoretic uncertainty principles to provide intuition into the properties of functions which linear group convolutional networks are likely to learn. 

\paragraph{Abelian groups} 
For abelian groups, the fundamental uncertainty principle details a trade-off between the norms of a function in its real and Fourier bases.

\begin{theorem}[Generalization of Donoho-Stark Theorem \citep{wigderson2021uncertainty}] 
\label{thm:abelian_uncertainty}
Given a finite abelian group $G$, let ${}\widehat{G}$ be the set of irreducible representations of $G$ and $f:G \to \mathbb{C}$ be a function mapping group elements to complex numbers. Let $\vf$ be the vectorized function and $\cF$ be the unitary group Fourier transform (see \autoref{def:group_fourier_transform}), then
\begin{equation}
    \frac{\|\vf\|_1}{\|\vf\|_\infty} \frac{\| \mathcal{F} \vf \|_1}{\| \mathcal{F} \vf \|_\infty} \geq |G|
\end{equation}
\end{theorem}

\begin{remark}
Since the size of the support of a vector is bounded by $|\text{supp}(\vv)| \geq \frac{ \|\vv\|_1 }{ \|\vv\|_\infty }$, this directly implies that $|\text{supp}(\vf)| \; |\text{supp}(\mathcal{F} \vf)| \geq |G|$ recovering the Donoho-Stark theorem \citep{donoho1989uncertainty,matusiak2004donoho}.
\end{remark}

\paragraph{Non-abelian groups}
Since non-abelian groups have matrix valued irreducible representations, uncertainty theorems must account for norms and notions of support in the context of matrices. Here, we will provide two different uncertainty theorems for the non-abelian setting -- one via the rank of irreducible representations and another via the Schatten norm of irreducible representations. Uncertainty relations for non-abelian groups make use of the matrix group Fourier transform detailed in \autoref{def:group_fourier_transform}.

\begin{theorem}[Meshulam uncertainty theorem \citep{meshulam1992uncertainty}]
Given a finite non-abelian group $G$, let ${}\widehat{G}$ be the set of irreducible representations of $G$ and $f:G \to \mathbb{C}$ be a function mapping group elements to complex numbers. Let $\vf$ be the vectorized function and $\cF_M$ be the matrix group Fourier transform (see \autoref{def:group_fourier_transform}), then
\begin{equation}
    |\on{supp}(\vf)|\; \on{rank}(\mathcal{F}_M \vf) = |\on{supp}(\vf)|\; \left[ \sum_{\rho \in {}\widehat{G}} d_{\rho} \on{rank}\left({}\widehat{f}(\rho) \right) \right] \geq |G|
\end{equation}
\end{theorem}

The above theorem shows that the rank of the matrix Fourier transformed function is the proper notion of support in the uncertainty theorem for a non-abelian group. A stronger uncertainty principle which is a more direct corollary to \autoref{thm:abelian_uncertainty} can be obtained via the Schatten norms of the irreducible representations as shown below.

\begin{theorem}[Kuperberg uncertainty theorem \citep{wigderson2021uncertainty}]
Given a finite non-abelian group $G$, let ${}\widehat{G}$ be the set of irreducible representations of $G$ and $f:G \to \mathbb{C}$ be a function mapping group elements to complex numbers. Let $\vf$ be the vectorized function and $\cF_M$ be the matrix group Fourier transform (see \autoref{def:group_fourier_transform}), then
\begin{equation}
    \frac{\|\vf\|_1}{\|\vf\|_\infty}  \frac{\|\mathcal{F}_M \vf\|_1^{(S)}}{\|\mathcal{F}_M \vf\|_\infty^{(S)}} = \frac{\|\vf\|_1}{\|\vf\|_\infty}  \frac{\sum_{\rho \in {}\widehat{G}} d_{\rho} \|{}\widehat{f}(\rho)\|_1^{(S)}}{\max_{\rho \in {}\widehat{G}} \|{}\widehat{f}(\rho)\|_{\infty}^{(S)}}  \geq |G|.
\end{equation}
\end{theorem}

\section{Visualizing the implicit bias}
\label{app:visualizing_bias}

Implicit biases induced by the G-CNN architectures studied here can be readily observed by analyzing coefficients of the linearized transformation in the Real or Fourier regimes. Here, we visualize the linearized outputs a 3-layer linear G-CNN over the Dihedral group $D_{60}$ which has $4$ scalar irreps and $14$ irreps of dimension $2$ (hence $2 \times 2$ matrices). \autoref{fig:d60_viz_linearization} shows these linearized coefficients of the G-CNN, CNN, and FC at intialization and training.

As evident in \autoref{fig:d60_viz_linearization}, the learned coefficients of the G-CNN are sparse in the Fourier regime of the group. This sparsity pattern appears over blocks of irreps of length four, corresponding to coefficients of the $2 \times 2$ irreps of $D_{60}$. Furthermore, the values of the coefficients within a block are roughly constant, highlighting the bias towards low rank irreps. The relative denseness of coefficients of the trained G-CNN in the real regime is inherent due to the uncertainty principles of group functions. Unlike the G-CNN, the fully connected network (FC) appears to have no bias towards sparseness in its coefficients in either the real or Fourier regime. On a related note, the cyclic group of CNNs share some of the same irreps as those of the G-CNN studied here. This may be one explanation for the partial sparsity patterns observed in the coefficients of the CNN in the Fourier regime.

\section{Computational details}
\label{appa:comp}
As described in \autoref{sec:experiments}, for all models we use three-layer networks over $\mathbb{R}^{|G|}$ unless otherwise stated and binary classification tasks trained via standard gradient descent with exponential loss. We often train networks on isotropic Gaussian data points which are random vectors whose entries are drawn i.i.d. from a standard Normal distribution. For the linear networks on simulated data (with a fully connected output layer) we use the groups $D_8$ and $\left(C_5 \times C_5\right) \rtimes Q_8$. For the linear and nonlinear networks on MNIST, we use $(C_{28} \times C_{28}) \times D_8$. For the networks with ReLU activations and a linear layer (instead of pooling) we use the groups $D_8$ and $D_{60}$. For the experiments on MNIST with non-linear networks, we have used the \texttt{e2cnn} package~\cite{cesa2019e2cnn}. The weights are initialized with the standard uniform initialization. We choose an appropriate learning rate for each task depending on the dimension and magnitude of the values. Due to the choice of the exponential loss as our loss function, we sometimes periodically increased the learning rate since gradients decay over time to speed up convergence. Since each problem is overparameterized the loss will almost surely converge to zero, so we choose enough training epochs to achieve satisfactory convergence. For each plot, we report $95\%$ confidence intervals over $3$ to $50$ runs, depending on the classification task. The computational resources used are modest - commodity hardware should suffice to fully reproduce our results. To replicate experiments, please visit our code repository.\footnote{Our code is available here: \url{https://github.com/kristian-georgiev/implicit-bias-of-linear-equivariant-networks}}

\subsection{Additional experiments - linear}

\begin{figure*}[h!]
 \captionsetup[subfigure]{aboveskip=-1pt,belowskip=-3pt}  
 \begin{subfigure}{0.49\textwidth}
    \includegraphics[width=\linewidth]{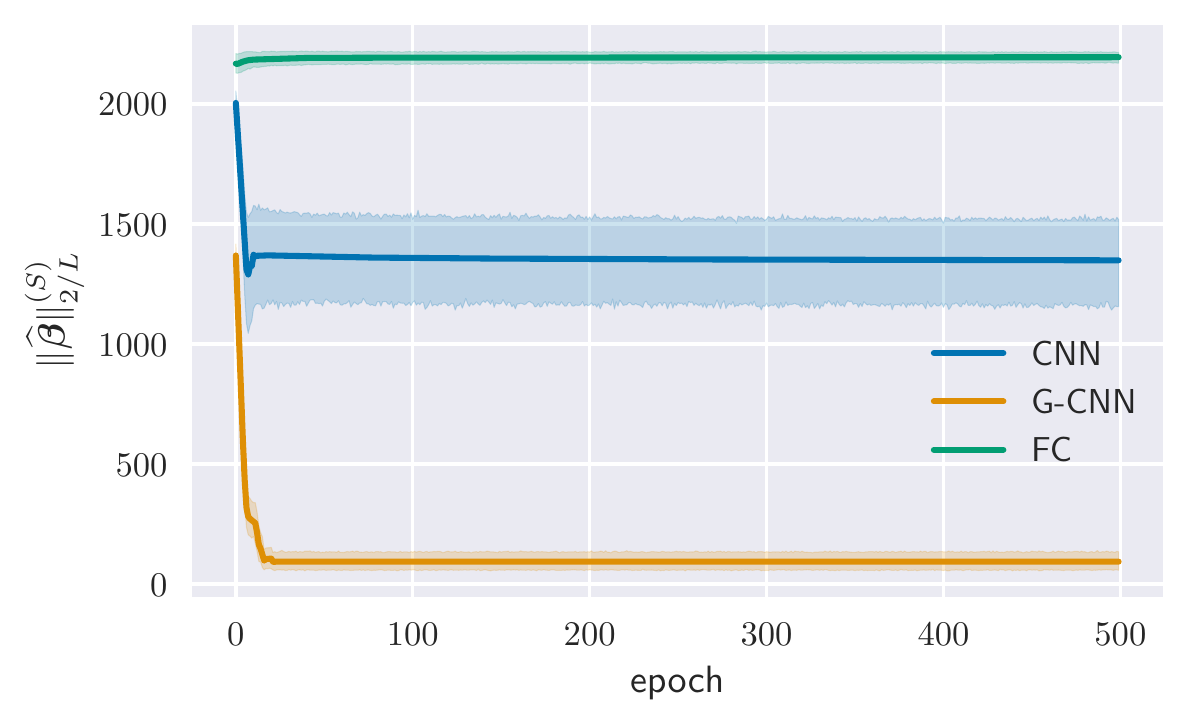}
    \caption{Fourier space norm of network linearization $\vbeta$} \label{fig:1a}
  \end{subfigure}%
  \hspace*{\fill}   
  \begin{subfigure}{0.49\textwidth}
    \includegraphics[width=\linewidth]{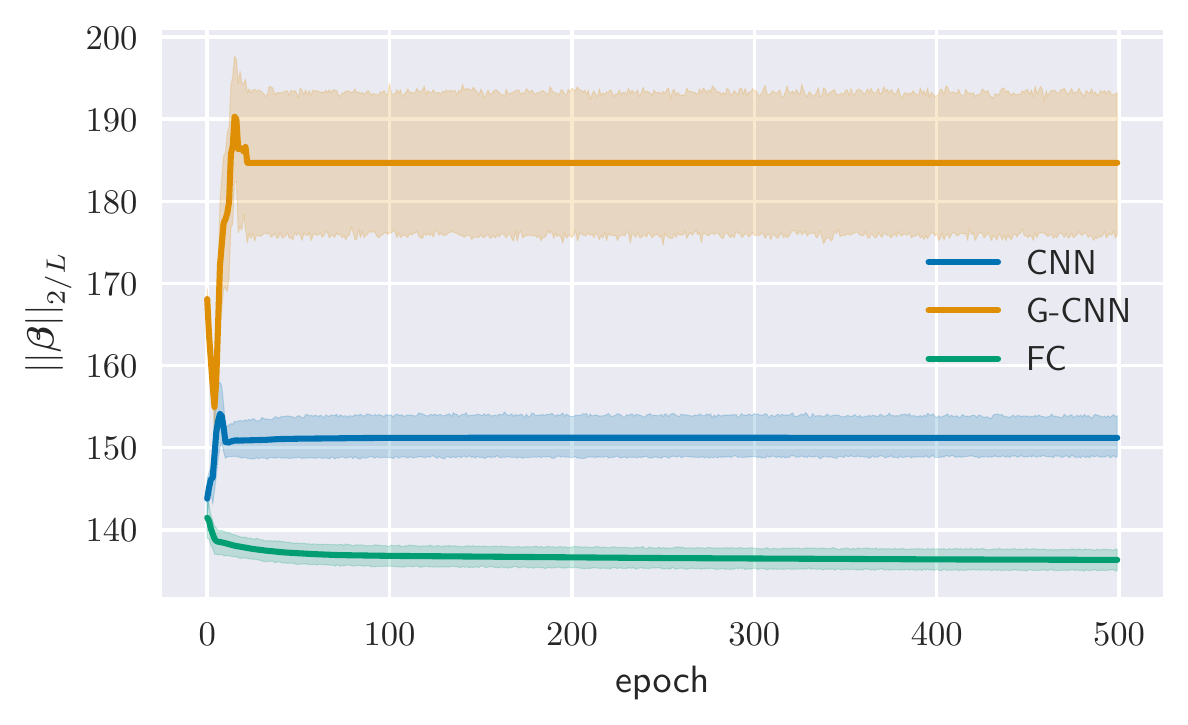}
    \caption{Real space norm of network linearization $\vbeta$} \label{fig:1b}
  \end{subfigure}%
  \hspace*{\fill}
\caption{Norms of the linearizations of three different linear architectures for the non-abelian group ${G=\left(C_5\times C_5\right) \rtimes Q_8}$ trained on a binary classification task with 10 isotropic Gaussian data points.} \label{fig:ccq}
\end{figure*}

We include linear architecture experiments on two additional groups. First, we include a G-CNN with the group $\left(C_5\times C_5\right) \rtimes Q_8$ (\autoref{fig:ccq}), which is a non-abelian group that has irreducible representations of up to dimension 8 and displays implicit regularization over a more elaborate group structure. Inputs are vectors with elements drawn i.i.d. from the standard Normal distribution. 

\subsection{Additional experiments - nonlinear}
We evaluate the implicit bias of an invariant ReLU G-CNN (with final pooling layer) with respect to translations, rotations, and flips on MNIST digits in \cref{fig:relu_MNIST}.

\begin{figure*}[t!]
    \captionsetup[subfigure]{aboveskip=-1pt,belowskip=-3pt}  
    \centering
    \includegraphics[width=0.6\linewidth]{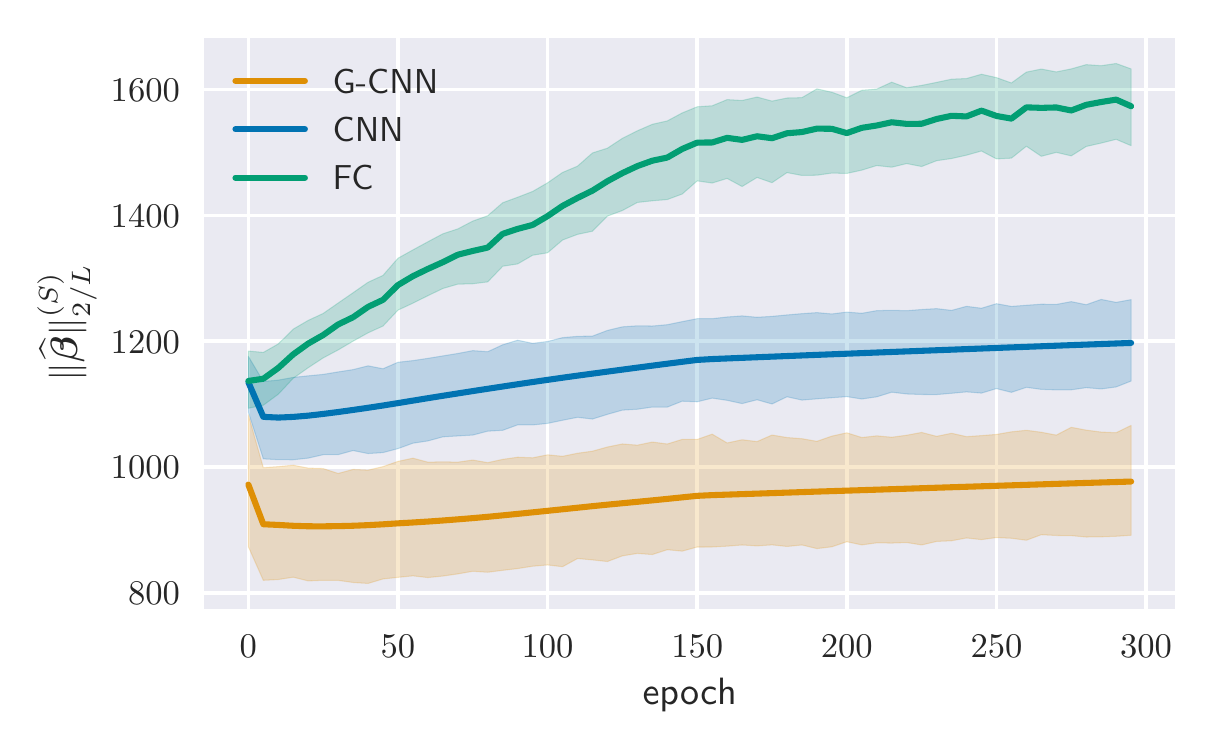}
\caption{Group Fourier norm for the non-abelian group $(C_{28}\times C_{28}) \times D_8$ for an additional \textit{nonlinear} architecture with ReLU activations shows that a nonlinear G-CNN seems to implicitly regularize locally. The figure tracks the mean norm of the per-sample local linearizations of each network. The network uses a pooling layer after the convolutional layers to maintain invariance. We evaluate a binary classification task using the MNIST digits $0$ and $5$.} \label{fig:relu_MNIST} \end{figure*}

\subsection{Spherical CNN Experimental Details}\label{app:spherical_details}

For the spherical CNN experiments, our architecture is drawn directly from \citet{cohen2018spherical}. Given an input image defined on a sphere, we apply to it a fixed $S_2$-convolution, followed by two learnable $SO(3)$-convolutions and a learnable fully-connected final layer, with ReLUs in-between. As in \citet{cohen2018spherical}, convolutions are done in Fourier space. We bandlimit the signals up to $\ell=5$, and then $\ell=3$. The first $SO(3)$-convolution has $20$ channels, and the second one has $100$. The use of varying bandlimits is adapted directly from \citet{cohen2018spherical}, but departs from our theory in the linear case of infinite groups. We compute all Schatten norms with respect to the bandwidth $\ell=5$. Since we use multiple ($20$) channels for the first convolution, we compute the linearization by computing the linearization for each channel separately, and then reporting the average Schatten norms over all channels. We note that this also departs from our theory towards a more realistic scenario, and we only provide these empirical observations for multi-channel networks.

The architecture is trained via stochastic gradient descent on the \emph{cross-entropy loss}, as the exponential loss did not result in learning the training data. Although this further departs from the theory, other implicit regularization works have studied the cross-entropy and exponential losses simultaneously, suggesting that the same results may carry over. Indeed, it is promising that the Schatten norm regularization occurs with a different loss function in the nonlinear case. For comparison architectures, we use a fully connected architecture and a 2D $CNN$ (applied to each discretization index of $SO(3)$). All architectures have as input a fixed, shared $S_2$ convolution and ReLU layer, which we conceptualize as a fixed, non-linear feature embedding. The reason for this is as follows: to apply the representation theory of $SO(3)$, we need inputs and linearization which are defined on $SO(3)$, not the homogeneous space $S_2$ (the sphere). As the output of an $S_2$ convolution is defined on $SO(3)$, the fixed feature embedding serves this purpose. 

\subsection{Omitted real-space plots}


\begin{figure}[h!]
    \centering
    \includegraphics[width=0.65\linewidth]{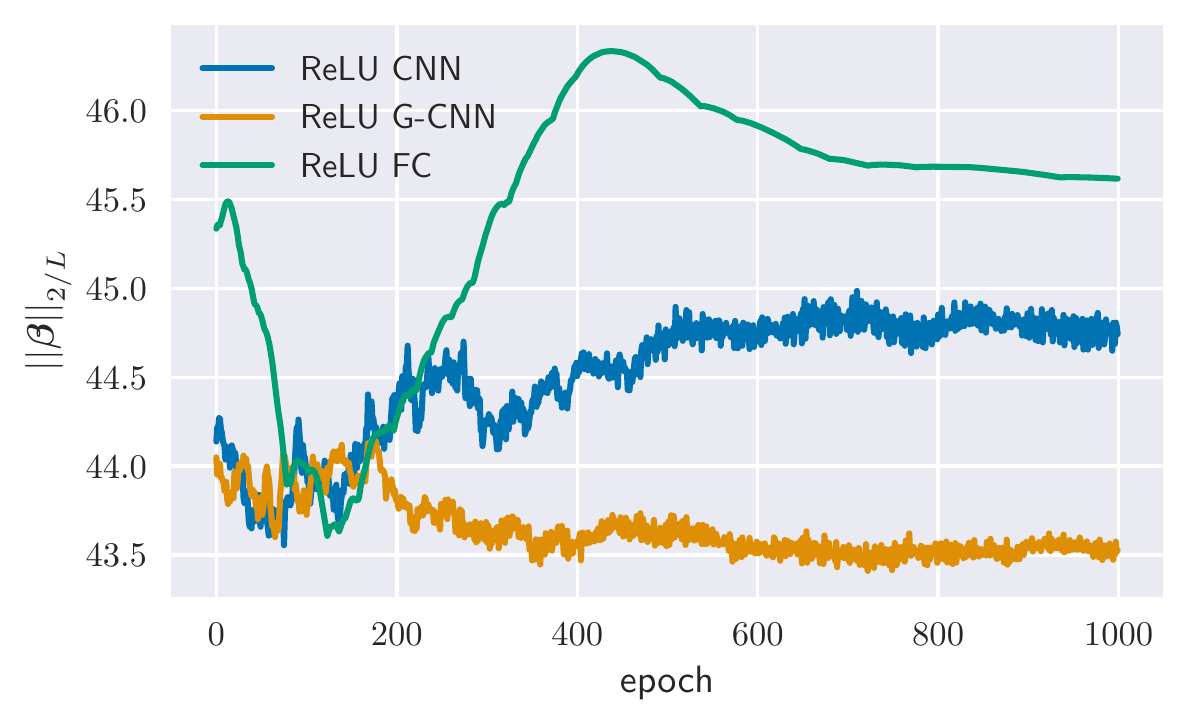}
    \caption{Real-space norm for $D_{60}$ with ReLU networks, 10 Gaussian training points. See \autoref{fig:relu_figb_d60} for comparable plot of norms in Fourier space.}
    \label{fig:loss_relu_d60_gaussian_10_sep_real_space}
\end{figure}

\subsection{Loss plots}
\begin{figure}[h!]
    \centering
    \includegraphics[width=0.65\linewidth]{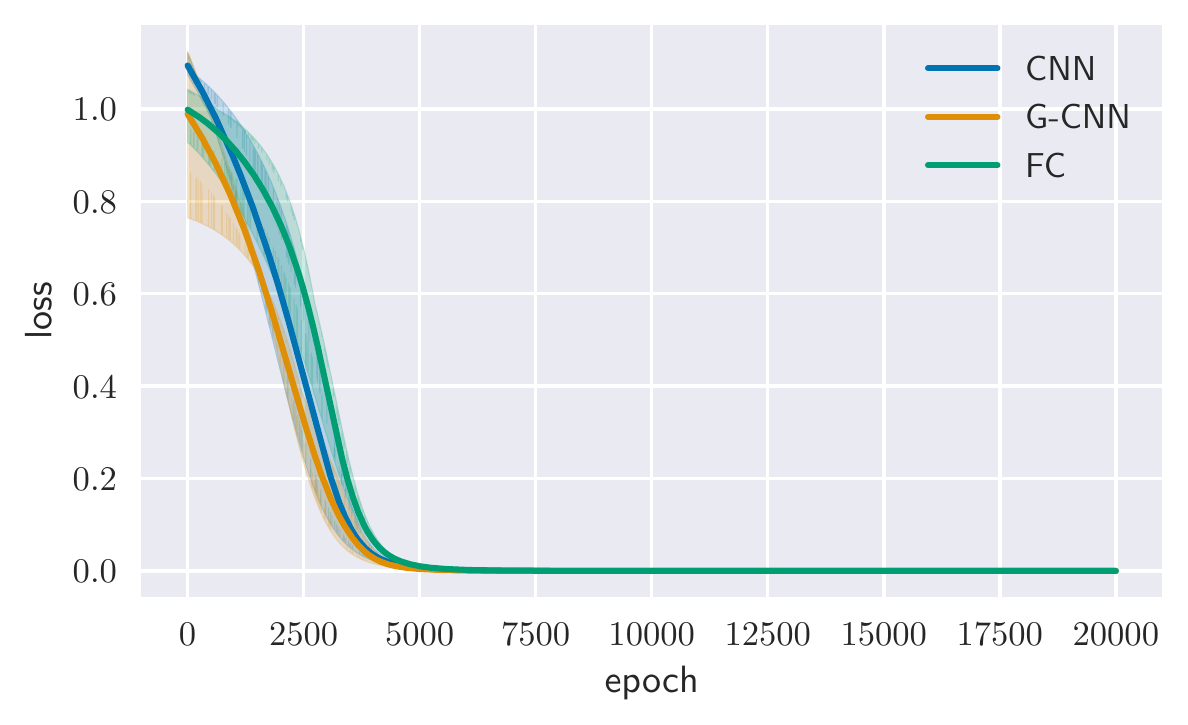}
    \caption{Loss trajectory for the setting of $D_8$ (see \autoref{fig:d8_gaussian_2}). Networks are trained on 2 Fourier i.i.d. data points.} 
    \label{fig:loss_d8_gaussian}
\end{figure}

\begin{figure}[h!]
    \centering
    \includegraphics[width=0.65\linewidth]{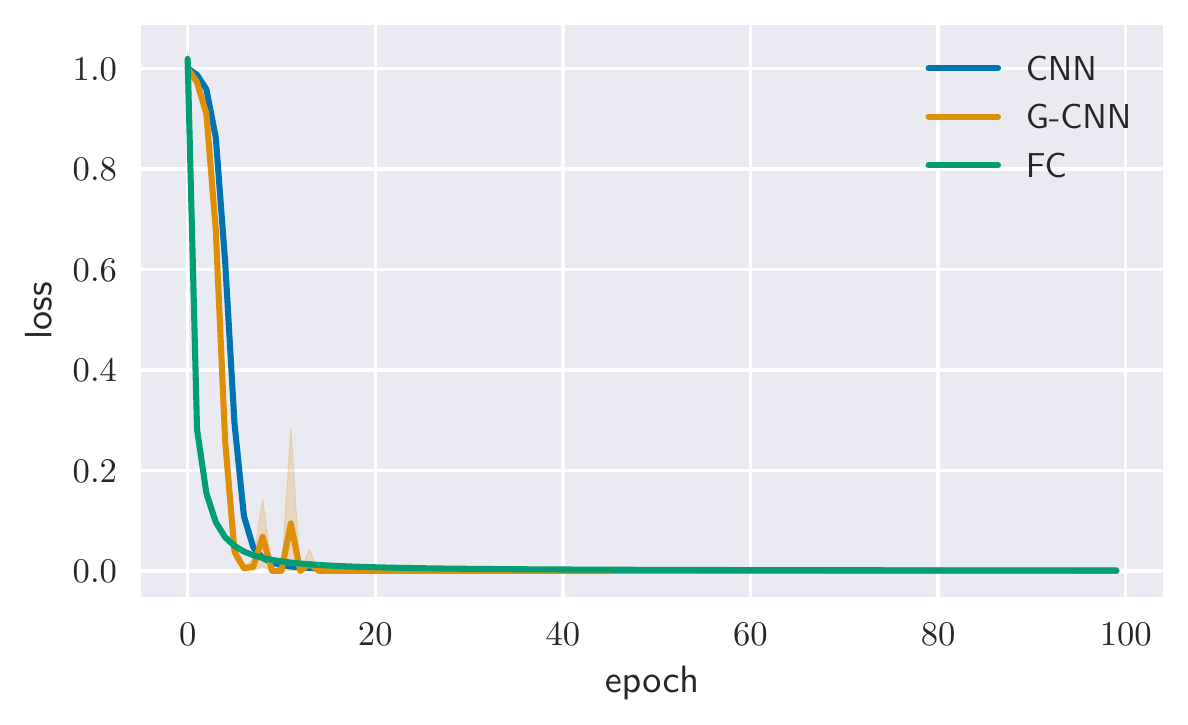}
    \caption{Loss trajectory for the setting of $C_{10}\times C_{10} \times C_2$ (see \autoref{fig:c10c10c2}). Networks are trained on 6 Gaussian i.i.d. data points.}
    \label{fig:loss_c10c10c2}
\end{figure}

\begin{figure}[h!]
    \centering
    \includegraphics[width=0.65\linewidth]{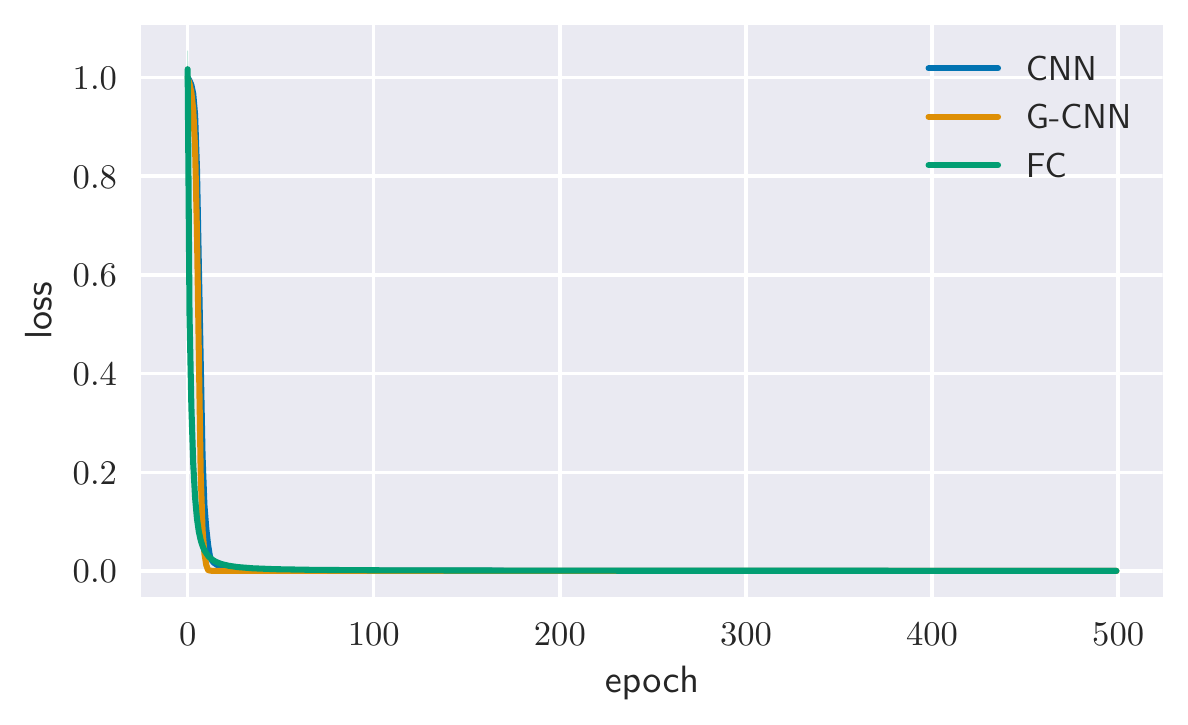}
    \caption{Loss trajectory for the setting of $\left(C_5\times C_5\right) \rtimes Q_8$ (see \autoref{fig:ccq}). Networks are trained on 10 Gaussian i.i.d. data points.}
    \label{fig:loss_small_boi_1}
\end{figure}

\begin{figure}[h!]
    \centering
    \includegraphics[width=0.65\linewidth]{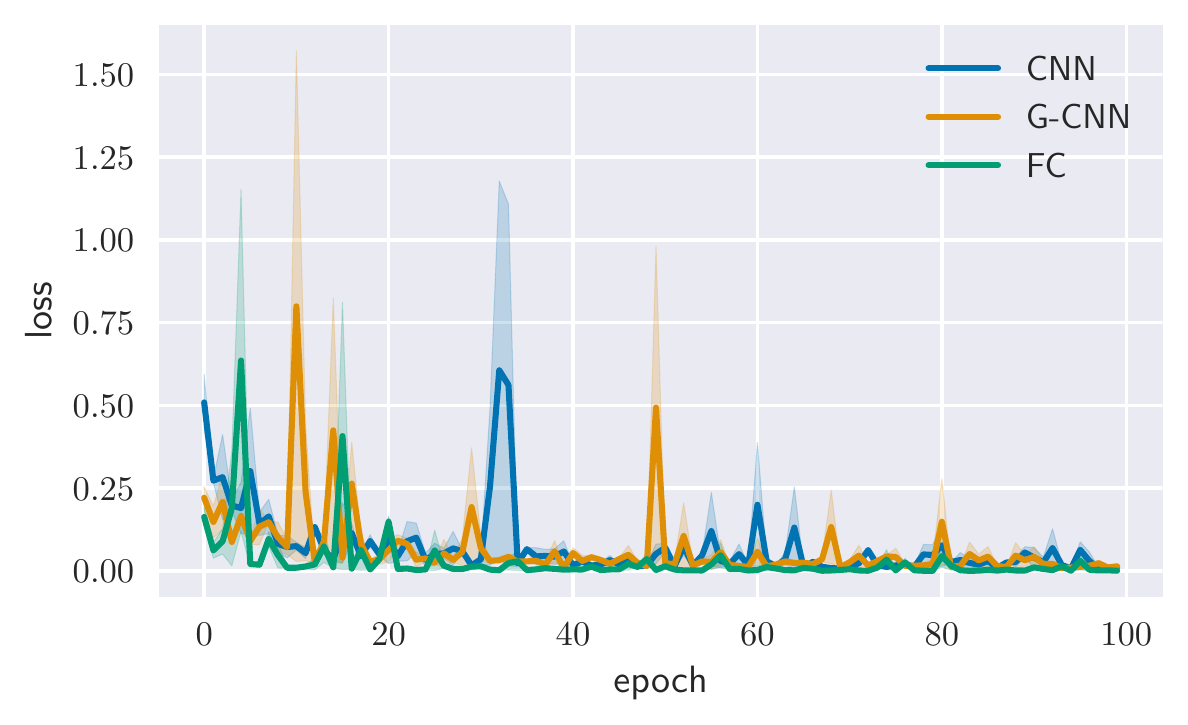}
    \caption{Loss trajectory for the setting of $\left(C_{28}\times C_{28}\right) \times D_8$ (see \autoref{fig:MNIST_linear}). Networks are trained on the digits 1 and 5 from MNIST.}
    \label{fig:loss_MNIST_linear}
\end{figure}


\begin{figure}[h!]
    \centering
    \includegraphics[width=0.65\linewidth]{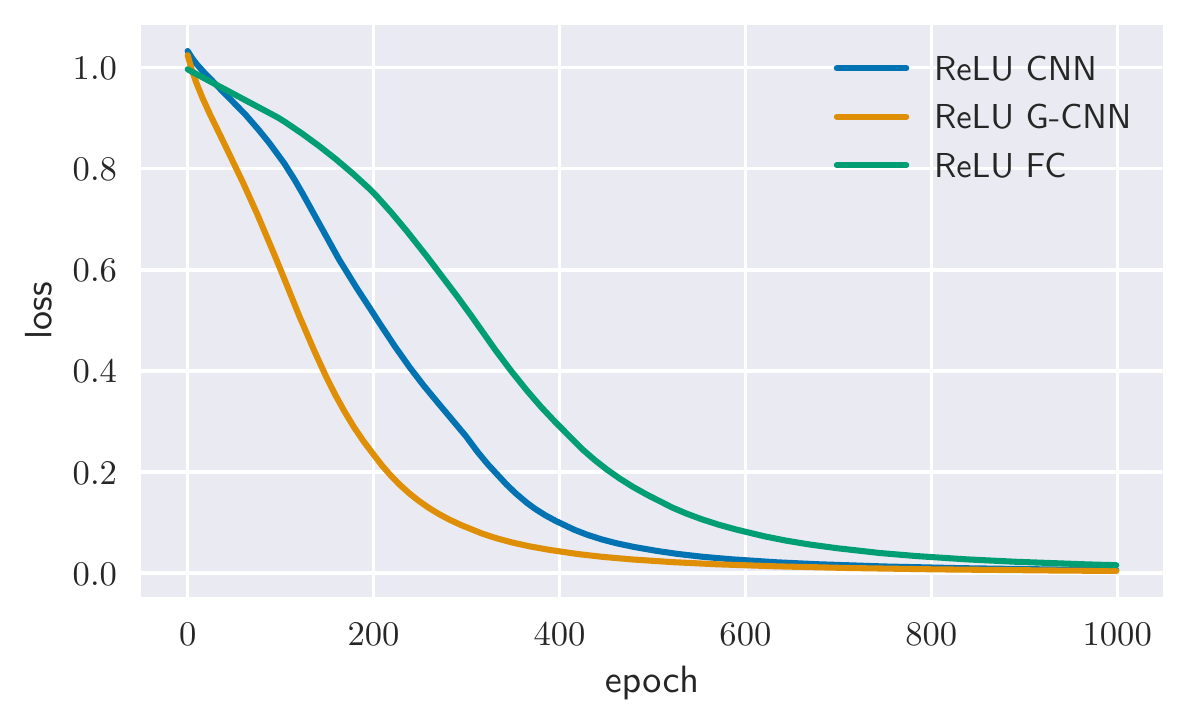}
    \caption{Loss trajectory for the setting of $D_{60}$ (see \autoref{fig:relu_figb_d60}). Networks are nonlinear with ReLU activations and trained on 10 distinct frequencies as data points.}
    \label{fig:loss_small_boi_3}
\end{figure}

\begin{figure}[h!]
    \centering
    \includegraphics[width=0.65\linewidth]{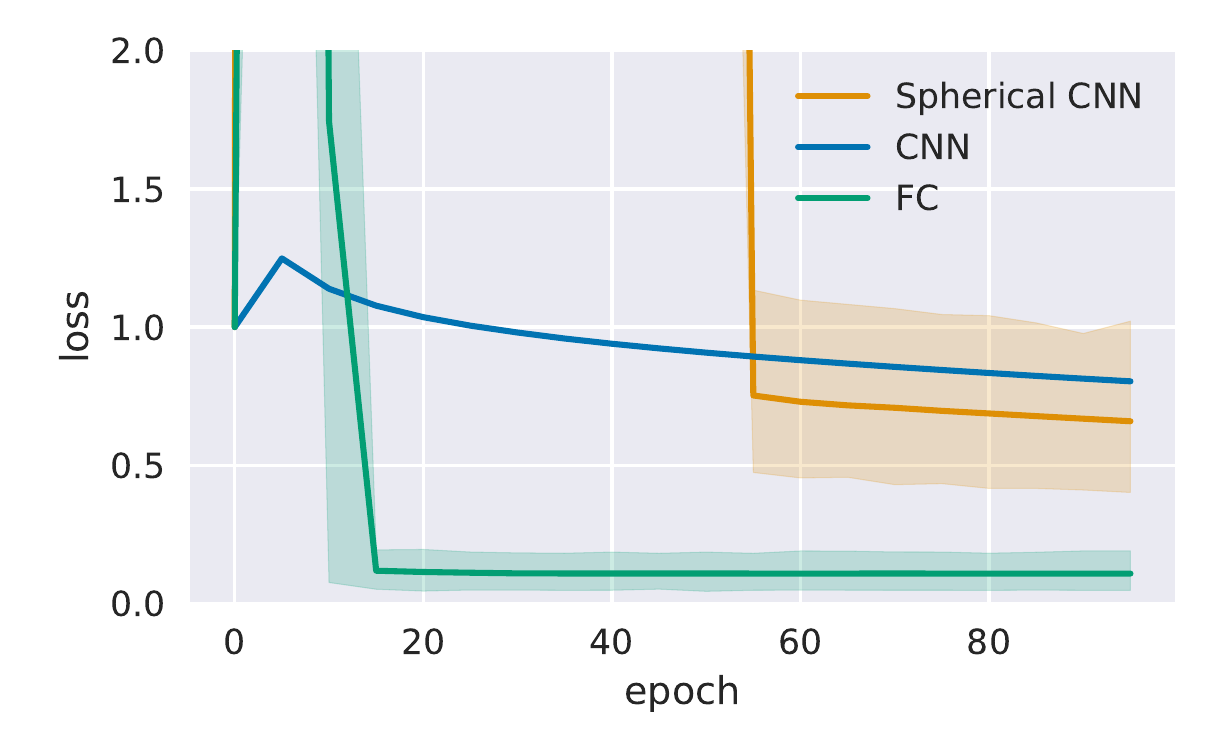}
    \caption{Loss trajectory for the setting of bandlimited $SO(3)$. Networks are nonlinear with ReLU activations and trained on a subset of spherical MNIST (digits $0$ and $5$).}
    \label{fig:loss_spherical}
\end{figure}

\end{document}